\documentclass[twoside,11pt]{article}

\usepackage{amsfonts,amsmath,url,xspace}
\usepackage{latexsym}
\usepackage{bm, subfigure}
\usepackage[retainorgcmds]{IEEEtrantools}
\usepackage{multirow}
\usepackage[left]{lineno}
\usepackage[margin=1.3in]{geometry}
\usepackage{natbib}
\usepackage{amsthm}
\usepackage{subfigure}
\usepackage{epsfig}
\usepackage{authblk}
\usepackage[ruled,noline,linesnumbered]{algorithm2e}
\newcommand{\argmax}{\operatornamewithlimits{argmax}}
\newcommand{\tr}{\operatorname{tr}}

\newcommand{\vct}[1]{\mathbf{#1}}
\newcommand{\vect}{\operatorname{vec}}
\newcommand{\sign}{\operatorname{sign}}
\newcommand{\grad}{\operatorname{grad}}

\newcommand{\lnormg}[2][\Lambda] {\|{#2}\|_{1,{#1}}}

\def\R{\mathbb{R}}
\def\S{\mathcal{S}}
\def\N{\mathcal{N}}
\def\Soft{\mathcal{S}}

\def\PQN{{\sc PQN}\xspace}
\def\QUIC{{\sc QUIC}\xspace}
\def\SINCO{{\sc SINCO}\xspace}
\def\PSM{{\sc PSM}\xspace}
\def\VSM{{\sc VSM}\xspace}

\def\glasso{{\sc glasso}\xspace}
\def\ALM{{\sc ALM}\xspace}
\def\ADMM{{\sc ADMM}\xspace}
\def\IPM{{\sc IPM}\xspace}
\def\bb{{\boldsymbol b}}
\def\bw{{\boldsymbol w}}
\def\bx{{\boldsymbol x}}

\newtheorem{definition}{Definition}
\newtheorem{lemma}{Lemma}
\newtheorem{proposition}{Proposition}
\newtheorem{theorem}{Theorem}

\begin{document}


\title{Sparse Inverse Covariance Matrix Estimation \\
  Using Quadratic Approximation}

\author[1]{Cho-Jui Hsieh}
\author[1]{M\'{a}ty\'{a}s A. Sustik} 
\author[1]{Inderjit S. Dhillon}
\author[1]{Pradeep Ravikumar}
\affil[1]{Department of Computer Science, University of Texas, Austin}

\renewcommand\Authands{ and }
\date{}

\maketitle

\begin{abstract}
  The $\ell_1$-regularized Gaussian maximum likelihood estimator (MLE) has been shown to have strong statistical guarantees in recovering a sparse inverse   covariance matrix, or alternatively the underlying graph structure of a   Gaussian Markov Random Field, from very limited samples. We propose a novel   algorithm for solving the resulting optimization problem which is a   regularized log-determinant program.  In contrast to recent state-of-the-art   methods that largely use first order gradient information, our algorithm is   based on Newton's method and employs a quadratic approximation, but with   some modifications that leverage the structure of the sparse Gaussian MLE   problem.  We show that our method is superlinearly convergent, and present experimental results using synthetic and real-world application data that   demonstrate the considerable improvements in performance of our method when   compared to other state-of-the-art methods.
\end{abstract}

\section{Introduction}

Statistical problems in modern data settings are increasingly high-dimensional, where the number of parameters is very large, potentially outnumbering even the number of observations. An important class of such problems involves estimating the graph structure of a Gaussian Markov random
field (GMRF), with applications ranging from biological inference in gene networks, analysis of  
fMRI brain connectivity data and analysis of interactions in social networks.
Specifically, given $n$ independently drawn samples 
$\{\vct{y_1}, \vct{y_2}, \dots, \vct{y_n}\}$ from a $p$-variate Gaussian 
distribution, so that $\vct{y_i} \sim \mathcal{N}(\boldsymbol\mu, \Sigma)$, 
the task is to estimate its inverse covariance matrix $\Sigma^{-1}$, also 
referred to as the {\em precision} or {\em concentration} matrix. The non-zero 
pattern of this inverse covariance matrix $\Sigma^{-1}$ can be shown to 
correspond to the underlying graph structure of the GMRF.  An active line of 
work in high-dimensional settings, where $p \gg n$, is based on imposing 
constraints on the model space; 
in the GMRF case a common structured constraint is that of sparsity of the 
inverse covariance matrix. 
Accordingly, recent papers 
by~\cite{Banerjee:glasso:2008,Friedman:glasso:2007,Yuan-Lin:model-selection} 
have proposed an estimator that minimizes the Gaussian negative log-likelihood 
regularized by the $\ell_1$ norm of the entries (typically restricted to those 
on the off-diagonal) of the inverse covariance matrix,  
which encourages sparsity in its entries. The resulting 
optimization problem is a log-determinant program, which is convex, and can be 
solved in polynomial time.

For such large-scale optimization problems arising from high-dimensional statistical estimation, standard optimization methods typically suffer sub-linear rates of convergence~(\cite{ANW10}). This would be too expensive for the Gaussian MLE problem, since the number of matrix entries scales quadratically with the number of nodes. Luckily, the log-determinant problem has special structure; the log-determinant function is strongly convex and one can thus obtain linear (i.e. geometric) rates of convergence via the state-of-the-art methods.  However, even linear rates in turn become infeasible when the problem size is very large, with the number of nodes in the thousands and the number of matrix entries to be estimated in the millions.  Here we ask the question: \textit{can we obtain superlinear   rates of convergence for the optimization problem underlying the $\ell_1$-regularized Gaussian MLE?}

For superlinear rates, one has to consider second-order methods which at least in part use the Hessian of the objective function. There are however some caveats to the use of such second-order methods in high-dimensional settings. First, a straightforward implementation of each second-order step would be very expensive for high-dimensional problems.  Secondly, the log-determinant function in the Gaussian MLE objective acts as a barrier function for the positive definite cone. This barrier property would be lost under quadratic approximations so there is a danger that Newton-like updates will not yield positive-definite matrices, unless one explicitly enforces such a constraint in some manner.

In this paper, we present {\sc QUIC} (QUadratic approximation of Inverse Covariance matrices), a second-order algorithm, that solves the $\ell_1$-regularized Gaussian MLE. We perform Newton steps that use iterative quadratic approximations of the Gaussian negative log-likelihood.  The computation of the Newton direction is a {\em Lasso} problem~(\cite{Mei08,GLMNET}), which we then solve using coordinate descent. A key facet of our method is that we are able to reduce the computational cost of a coordinate descent update from the naive $O(p^2)$ to $O(p)$ complexity by exploiting the structure present in the problem, and by a careful arrangement and caching of the computations. Furthermore, an Armijo-rule based step size selection rule ensures sufficient descent and positive definiteness of the intermediate iterates. Finally, we use the form of the stationary condition characterizing the optimal solution to \emph{focus} the Newton direction computation on a small subset of \emph{free} variables, but in a manner that preserves the strong convergence guarantees of second-order descent. We note that when the solution has a block-diagonal structure as described in \cite{hastie:2012,friedman:2011}, the {\em fixed}/{\em free} set selection in \QUIC can automatically identify this sparsity structure and avoid updates to the off-diagonal block elements.

A preliminary version of this paper appeared in the NIPS 2011 conference \citep{HSDR:NIPS2011_1249}. 
This paper contains the following enhancements to the conference version: 
(i) inclusion of generalized matrix regularized in the problem statement and subsequent analysis,
(ii) detailed convergence analysis of our proposed method, QUIC, 
(iii) detailed experimental comparisons under different parameters and precision/recall curves for synthetic data, and 
(iv) new material on the relationship of the {\em free} and {\em fixed} sets in QUIC with recent covariance thresholding techniques~(see Section 3.4 for more details).

The outline of the paper is as follows. We start with a review of related work 
and the problem setup in Section~\ref{sec:background}.  In 
Section~\ref{sec:quadratic-solver}, we present our algorithm that combines 
quadratic approximation, Newton's method and coordinate descent. In 
Section~\ref{sec:converge}, we show superlinear convergence of our method. 
We summarize the experimental results in 
Section~\ref{sec:experiments}, where we compare the algorithm using both real 
data and synthetic examples from~\cite{LL10a}.
We observe that our algorithm performs 
overwhelmingly better (quadratic instead of linear convergence) than 
existing solutions described in the literature.

{\sc Notation. } In this paper, boldfaced lowercase letters 
denote vectors and uppercase letters denote $p\times p$ real matrices. 
$\S^p_{++}$ denotes the space of $p\times p$ symmetric positive definite 
matrices and $X\succ 0$ and $X\succeq 0$ means that $X$ is positive definite 
or positive semidefinite, respectively.  The vectorized listing of the 
elements of a $p\times p$ matrix $X$ is denoted by $\vect(X) \in \R^{p^2}$ and 
the Kronecker product of the matrices $X$ and $Y$ is denoted by $X\otimes Y$.  
For a real-valued function $f(X)$, $\nabla f(X)$ is a $p\times p$ matrix with 
$(i,j)$ element equal to $\frac{\partial}{\partial X_{ij}} f(X)$ and denoted by $\nabla_{ij} f(X)$, while 
$\nabla^2 f(X)$ is the $p^2 \times p^2$ Hessian matrix. We will use the 
$\ell_1$ and $\ell_\infty$ norms defined on the vectorized form of matrix $X$: 
$\|X\|_1:=\sum_{i,j} |X_{ij}|$ and $\|X\|_\infty:=\max_{i,j} |X_{ij}|$. 
We also employ elementwise $\ell_1$-regularization,  
$\|X\|_{1,\Lambda}:=\sum_{i,j}\lambda_{ij}|X_{ij}|$, where 
$\Lambda=[\lambda_{ij}]$ with $\lambda_{ij}>0$ for off-diagonal elements, and 
$\lambda_{ii}\geq 0$ for diagonal elements. 

\section{Background and Related work} \label{sec:background}

Let $\vct{y}$ be a $p$-variate Gaussian random vector, with distribution $\mathcal{N}(\boldsymbol\mu, \Sigma)$. Given $n$ independently drawn samples $\{\vct{y_1}, \dots, \vct{y_n}\}$ of this random vector,  the sample covariance matrix can be written as 
\begin{eqnarray} \label{eqn:empirical-cov}
  S = \frac{1}{n-1}\sum_{k=1}^n (\vct{y_k} - \vct{\hat{\boldsymbol\mu}})(\vct{y_k} - \vct{\hat{\boldsymbol\mu}})^T,\ \mbox{where}\ {\vct{\hat{\boldsymbol\mu}}}=\frac{1}{n}\sum_{k=1}^n \vct{y_k}.
\end{eqnarray}
Given a regularization penalty $\lambda > 0$, the $\ell_1$-regularized 
Gaussian MLE for the inverse covariance matrix can be written as the solution 
of the following regularized \emph{log-determinant} program:
\begin{equation} \label{eqn:primal}
   \arg\min_{X \succ 0} \bigg\{-\log\det X + \text{tr}(SX) +
  \lambda \sum_{i,j=1}^p |X_{ij}| \bigg\}. 
\end{equation}
The $\ell_1$ regularization promotes sparsity in the inverse covariance matrix, and thus encourages a sparse graphical model structure.  We consider a generalized weighted $\ell_1$ regularization, where given a symmetric nonnegative weight matrix $\Lambda = [\lambda_{ij}]$, we can assign different nonnegative weights to different entries, obtaining the regularization term  $\|X\|_{1, \Lambda}=\sum_{i,j=1}^p \lambda_{ij}|X_{ij}|$. 
In this paper we  will focus on solving the following generalized sparse inverse covariance  estimation problem: 
\begin{equation} \label{eqn:primal_general}
  \arg\min_{X \succ 0} \bigg\{-\log\det X + \text{tr}(SX) +
  \|X\|_{1,\Lambda} \bigg\} = \arg\min_{X \succ 0} f(X), 
\end{equation}
where $X^*=(\Sigma^*)^{-1}$. 
In order to ensure that problem \eqref{eqn:primal_general} has a unique 
minimizer, as we show later, it is sufficient to require that $\lambda_{ij}>0$  for off-diagonal 
entries, and $\lambda_{ii}\ge 0$ for diagonal entries. 
The standard  off-diagonal $\ell_1$ regularization variant 
$\lambda\sum_{I \neq j}|X_{ij}|$  is a special case of this weighted 
regularization function.  For further  details on the background and utility 
of $\ell_1$ regularization in the  context of GMRFs, we refer the reader 
to~\cite{Yuan-Lin:model-selection, Banerjee:glasso:2008, Friedman:glasso:2007, 
RWRY11, JD08a}.

Due in part to its importance, there has been an active line of work on 
efficient optimization methods for solving \eqref{eqn:primal} and \eqref{eqn:primal_general}.
Since the regularization term is non-smooth and hard to 
solve, many methods aim to solve the dual problem of \eqref{eqn:primal_general}:
\begin{equation}
  \Sigma^* = \argmax_{ |W_{ij}-S_{ij}| \leq \lambda_{ij}} \log\det W,
  \label{eqn:dual}
\end{equation}
which has a smooth objective function with bounded constraints. 
\cite{Banerjee:glasso:2008} propose a block-coordinate descent method to solve 
the dual problem~\eqref{eqn:dual}, by updating one row and column of $W$ at a 
time. They show that the dual of the corresponding row subproblem can be 
written as a standard Lasso problem, which they then solve by Nesterov's first 
order method. \cite{Friedman:glasso:2007} follow the same strategy, but propose 
to use a coordinate descent method to solve the row subproblems instead; 
their method is implemented in the widely used R package called {\sc glasso}.  
In other approaches, the dual problem~\eqref{eqn:dual} is 
treated as a constrained optimization problem, for which~\cite{JD08a} apply a 
projected subgradient method called \PSM, while~\cite{ZL09a} proposes an 
accelerated gradient descent method called \VSM.

Other first-order methods have been pursued to solve the primal optimization 
problem \eqref{eqn:primal}. \cite{DAspremont:2008} apply Nesterov's first 
order method to~\eqref{eqn:primal} after smoothing the 
objective function; \cite{ALM} apply an augmented Lagrangian method to handle 
the smooth and nonsmooth parts separately; the resulting algorithm is 
implemented in the \ALM software package.  In~\cite{SINCO}, the authors propose 
to directly solve  the primal problem by a greedy coordinate descent method 
called \SINCO. However, each coordinate update of \SINCO has a time complexity 
of  
$O(p^2)$, which becomes computationally prohibitive when handling large problems. We 
will show in this paper that after forming the quadratic approximation, the 
time complexity of one coordinate update can be 
performed in $O(p)$ operations. 
This trick is one of the key advantages of our proposed method, \QUIC.

One common characteristic of the above methods is that they are first-order  
iterative methods that mainly use gradient information at each step. Such  
first-order methods have become increasingly popular in recent years for  
high-dimensional problems in part due to their ease of implementation, and  
because they require very little computation and memory at each step. The  
caveat is that they have at most linear rates of  
convergence~(\cite{Bertsekas95}). To achieve superlinear convergence rates,  
one has to consider second-order methods, which have attracted some  
attention only recently for the sparse inverse covariance estimation 
problem.  \cite{LL10a} handle the non-smoothness of the $\ell_1$ 
regularization in the  objective function by doubling the number of variables, 
and solving the  resulting constrained optimization problem by an inexact 
interior point method.  \cite{Schmidt:PQN:2009} propose a  second order 
Projected Quasi-Newton method (\PQN)  that solves the dual 
problem~\eqref{eqn:dual}, since the dual objective function is smooth.  
The key difference of our method  when compared to these recent second order 
solvers is that we directly solve  the $\ell_1$-regularized primal objective 
using a second-order method.  As we show, this allows us to  leverage structure 
in the problem, and efficiently approximate the generalized  Newton direction 
using coordinate descent. Subsequent to the preliminary  version of this  
paper (see \citep{HSDR:NIPS2011_1249}),~\cite{Newton-Quadratic:Nocedal} have 
proposed generalizations to our framework to allow various inner solvers such as 
FISTA,  conjugate gradient (CG), and LBFGS to be used, in addition to our proposed coordinate descent scheme.

\section{Quadratic Approximation Method} \label{sec:quadratic-solver}

We first note that the objective $f(X)$ in the non-differentiable optimization problem~\eqref{eqn:primal_general}, can be written as the sum of two parts, $f(X) = g(X)+h(X)$, where
\begin{equation} \label{eqn:composite}
  g(X) = -\log\det X + \tr(SX) \text{ and } h(X) = \lnormg{X}.
\end{equation}
The first component $g(X)$ is twice differentiable, and strictly convex.  The second part, $h(X)$, is convex but non-differentiable. Following the approach of \cite{PT07a} and  \cite{SY09a}, we build a quadratic approximation around any iterate $X_t$ for this composite function by first considering the second-order Taylor expansion of the smooth component $g(X)$:
\begin{equation}
  \bar{g}_{X_t} (\Delta) \equiv  g(X_t) + \vect(\nabla g(X_t))^T \vect(\Delta) + 
  \frac{1}{2}\vect(\Delta)^T \nabla^2 g(X_t) \vect(\Delta).
\end{equation}
The Newton direction $D_t$ for the entire objective $f(X)$ can then be written as the solution of the regularized quadratic program:
\begin{equation} \label{eqn:subpb}
  D_t = \arg\min_{\Delta} \big\{ \bar{g}_{X_t}(\Delta) + h(X_t + \Delta) \big\}.
\end{equation}
We use this Newton direction to compute our iterative estimates $\{X_t\}$ for 
the solution of the optimization problem~\eqref{eqn:primal_general}. This 
variant of Newton method for such composite objectives is also referred to 
as a ``proximal Newton-type method,''  and was empirically studied 
in~\cite{Marc:thesis}. \cite{PT07a} considered the more general case where the 
Hessian $\nabla^2 g(X_t)$ is replaced by any positive definite matrix. See 
also the recent paper by \cite{proximal-newton}, where convergence properties 
of such general proximal Newton-type methods are discussed. We note that a key 
caveat to applying such second-order methods in high-dimensional settings is 
that the computation of the Newton direction appears to have a large time complexity, 
which is one reason why first-order methods have been so popular for solving 
the high-dimensional $\ell_1$-regularized Gaussian MLE.

Let us delve into the Newton direction computation in ~\eqref{eqn:subpb}. Note that it can be rewritten as a standard Lasso regression problem~\citep{Lasso}:
\begin{equation} \label{eqn:lasso_subpb}
  \arg\min_{\Delta} \frac{1}{2}\|H^{\frac{1}{2}} \vect(\Delta) +
  H^{-\frac{1}{2}}\bb \|_2^2 + \lnormg{\vect(X_t+\Delta)},
\end{equation}
where $H=\nabla^2 g(X_t)$ and $\bb = \vect(\nabla g(X_t))$.  Many efficient 
optimization methods exist that solve Lasso regression problems, such as 
the coordinate descent method~\citep{Mei08}, the gradient projection 
method~\citep{BTP69a}, and iterative shrinking methods~\citep{ISTA, FISTA}.  
When applied to the Lasso problem of~\eqref{eqn:subpb}, most of these 
optimization methods would require the computation of the gradient of 
$\bar{g}_{X_t}(\Delta)$:
\begin{equation}
  \nabla \bar{g}_{X_t}(\Delta) = H \vect(\Delta)+ \bb.
\label{eqn:lasso_grad}
\end{equation}
The straightforward approach for computing~\eqref{eqn:lasso_grad} for a general $p^2\times p^2$ Hessian matrix $H$ would take $O(p^4)$ time, making it impractical for large problems. Fortunately, for the sparse inverse covariance problem~\eqref{eqn:primal_general}, the Hessian matrix $H$ has the following special form (see for instance~\cite[Chapter~A.4.3]{SB04a}):
\begin{equation*}
  H = \nabla^2 g(X_t) = X_t^{-1} \otimes X_t^{-1}.
\end{equation*}
We show how to exploit this special form of the Hessian matrix to perform one coordinate descent step that updates one element of $\Delta$ in $O(p)$ time. Hence a full sweep of coordinate descent steps over all the variables requires $O(p^3)$ time. This key observation is one of the reasons that makes our Newton-like method viable for solving the inverse covariance estimation problem.

There exist other functions which allow efficient computation of Hessian times vector. 
As an example, we consider the case of $\ell_1$-regularized logistic 
regression. Suppose we are given $n$ samples with feature vectors 
$\bx_1, \dots, \bx_n\in \R^p$ and labels $y_1, \dots, y_n$, and we solve the 
following optimization problem to compute the model parameter $\bw$:  
\begin{equation*}
  \arg\min_{\bw\in \R^p} \sum_{i=1}^n \log(1+e^{-y_i \bw^T \bx_i}) + \lambda 
  \|\bw\|_1.  
\end{equation*}
Following our earlier approach, we can decompose this objective function into smooth and non-smooth parts, $g(\bw)+h(\bw)$, where
\begin{equation*}
  g(\bw) = \sum_{i=1}^n \log(1+e^{-y_i \bw^T \bx_i}) \text{ and } h(\bw) = 
  \lambda \|\bw\|_1. 
\end{equation*}
In order to apply coordinate descent to solve the quadratic approximation, we have to compute the gradient as in~\eqref{eqn:lasso_grad}.  The Hessian matrix $\nabla^2 g(\bw)$ is a $p\times p$ matrix, so 
direct computation of this gradient costs $O(p^2)$ flops. However, the Hessian matrix for logistic regression has the following simple form
\begin{equation*}
  H = \nabla^2 g(\bw) = X D X^T, 
\end{equation*}
where $D$ is a diagonal matrix with $D_{ii}= \frac{e^{-y_i \bw^T
    \bx_i}}{(1+e^{-y_i \bw^T \bx_i})^2}$ and $X=[\bx_1, \ \bx_2, \ \dots,
  \ \bx_n]$. Therefore we can write
\begin{equation}
  \nabla \bar{g}_{X_t} (\Delta) = (\nabla^2 g(\bw)) \vect(\Delta)+ \bb = XD (X^T \vect(\Delta) )+\bb. 
  \label{eqn:hv_logistic}
\end{equation}
The time complexity to compute \eqref{eqn:hv_logistic} is only proportional to the number of nonzero elements in the data matrix $X$, which can be much smaller than $O(p^2)$ for sparse datasets. Therefore similar quadratic approximation approaches are also efficient for solving the $\ell_1$-regularized logistic regression problem as shown by~\citep{GLMNET,GXY11a}.

In the sequel, we detail three innovations which make our quadratic approximation algorithm feasible for solving large sparse inverse covariance problems.  First, 
we approximate the Newton direction computation using an efficient coordinate descent 
method that exploits the structure of Hessian matrix, so that we reduce the time cost of 
each coordinate descent update step 
from $O(p^2)$ to $O(p)$.  Second, we employ an Armijo-rule based step size selection to ensure sufficient descent \emph{and} positive-definiteness of the next iterate. Finally, we use the form of the stationary condition characterizing the optimal solution, to \emph{focus} the Newton direction computation to a small subset of \emph{free} variables, in a manner that preserves the strong convergence guarantees of second-order descent.  We outline each of these three innovations in the following three subsections. A high level overview of our method is presented in Algorithm~\ref{alg:quadratic-high-level}. 
\begin{algorithm}[htp!]
  \DontPrintSemicolon
  \caption{Quadratic Approximation method for Sparse Inverse Covariance
    Learning ({\sc QUIC} overview) \label{alg:quadratic-high-level}}
  \SetKwInOut{Input}{Input}\SetKwInOut{Output}{Output}
  \Input{Empirical covariance matrix $S$ (positive semi-definite, $p\times
    p$), regularization parameter matrix $\Lambda$, initial iterate $X_0$, inner
    stopping tolerance $\epsilon$. }
  \Output{Sequence $\{X_t\}$ that converges to $\arg\min_{X \succ 0}
    f(X)$, where $f(X) = g(X) + h(X)$, where $g(X)=-\log\det X + 
    \text{tr}(SX), h(X)=\lnormg{X}$.}
  \For{$t = 0,1,\ldots$}{
    Compute $W_t = X_t^{-1}$. \;
    Form the second order approximation
    $\bar{f}_{X_t}(\Delta):=\bar{g}_{X_t}(\Delta)+h(X_t+\Delta)$ to
    $f(X_t+\Delta)$. \;
    Partition the variables into free and fixed sets based on the
    gradient, see Section~\ref{sec:sparsity}. \;
    Use coordinate descent to find the Newton direction $D_t = \arg
    \min_{\Delta} \bar{f}_{X_t} (X_t + \Delta)$ over the set of free variables,
    see~\eqref{eqn:one-var} and~\eqref{eqn:one-update} in Section \ref{sec:subpb}. (A {\em Lasso}
    problem.) \;
    Use an {\em Armijo}-rule based step-size selection to get $\alpha$
    such that $X_{t+1} = X_t + \alpha D_t$ is positive definite and there is sufficient 
    decrease in the objective function,
    see~\eqref{eqn:armijo} in Section \ref{sec:step-size}. \;
  }
\end{algorithm}

\subsection{Computing the Newton Direction} \label{sec:subpb}

In order to compute the Newton direction, we have to solve the Lasso problem~\eqref{eqn:subpb}. From~\citep[Chapter~A.4.3]{SB04a}, the gradient and Hessian for $g(X)=-\log\det X + \tr(SX)$ are
\begin{equation}
  \nabla g(X) = S- X^{-1} \text{ and } \nabla^2 g(X) = X^{-1} \otimes X^{-1}.  
\end{equation}
In order to formulate our problem accordingly, we can verify that for a 
symmetric matrix $\Delta$ we have 
$\tr(X_t^{-1} \Delta X_t^{-1} \Delta) = \vect(\Delta)^T (X_t^{-1}\otimes X_t^{-1})\vect(\Delta)$, so that $\bar{g}_{X_t}(\Delta)$ in~\eqref{eqn:subpb} can be rewritten as
\begin{equation} \label{eqn:approx-g}
  \bar{g}_{X_t} (\Delta) = - \log\det X_t + \tr(SX_t) + \tr((S-W_t)^T \Delta ) + \frac{1}{2}\tr(W_t \Delta 
  W_t \Delta), 
\end{equation}
where $W_t = X_t^{-1}$.

In~\cite{JF07a}, \cite{TTW08a}, the authors show that coordinate descent methods are very efficient for solving Lasso type problems. An obvious way to update each element of $\Delta$ (to solve~\eqref{eqn:subpb}) requires $O(p^2)$ floating point operations since $W_t\otimes W_t$ is a $p^2 \times p^2$ matrix, thus yielding an $O(p^4)$ procedure. As we show below, our implementation reduces the cost of one variable update to $O(p)$ by exploiting the structure of the specific form of the second order term $\tr(W_t \Delta W_t \Delta)$.

For notational simplicity, we will omit the iteration index $t$ in the 
derivations below where we only discuss a single Newton iteration; this applies to 
the rest of the this section and section~\ref{sec:step-size} as well.  
(Hence, the notation for $\bar{g}_{X_t}$ is also simplified to $\bar{g}$.) 
Furthermore, we omit the use of a separate index for the coordinate descent 
updates.  Thus, we simply use $D$ to denote the current iterate approximating 
the Newton direction and use $D'$ for the updated direction. Consider the 
coordinate descent update for the variable $X_{ij}$, with $i < j$ that 
preserves symmetry: 
$D' = D + \mu (\vct{e}_i\vct{e}_j^T + \vct{e}_j\vct{e}_i^T)$.  The solution of 
the one-variable problem corresponding to~\eqref{eqn:subpb} yields $\mu$:
\begin{eqnarray} \label{eqn:one-var}
  \arg\min_{\mu} &\bar{g}(D + \mu(\vct{e}_i\vct{e}_j^T +
  \vct{e}_j\vct{e}_i^T)) + 2\lambda_{ij}|X_{ij} + D_{ij} + \mu|.
\end{eqnarray}
We expand the terms appearing in the definition of $\bar{g}$ after substituting $D' = D+\mu (\vct{e}_i\vct{e}_j^T + \vct{e}_j\vct{e}_i^T)$ for $\Delta$ in~\eqref{eqn:approx-g} and omit the terms not dependent on $\mu$.  The contribution of $\tr(SD') - \tr(WD')$ yields $2\mu(S_{ij} - W_{ij})$, while the regularization term contributes $2\lambda_{ij}|X_{ij} + D_{ij} + \mu|$, as seen from~\eqref{eqn:one-var}.  The quadratic term can be rewritten (using the fact that $\tr(AB) = \tr(BA)$ and the symmetry of $D$ and $W$) to yield:
\begin{eqnarray}
  \tr(WD'WD') = \tr(WDWD) & + & 4\mu\vct{w}_i^T D\vct{w}_j + 2\mu^2 (W_{ij}^2
  + W_{ii}W_{jj}),
\end{eqnarray}
where $\vct{w}_i$ refers to the $i$-th column of $W$. 
In order to compute the single variable update we seek the minimum of the following quadratic function of $\mu$:
\begin{eqnarray} \label{eqn:one-variable}
  \frac{1}{2}(W_{ij}^2 + W_{ii}W_{jj})\mu^2 + (S_{ij} - W_{ij} + \vct{w}_i^T
  D \vct{w}_j)\mu + \lambda_{ij}|X_{ij} + D_{ij} + \mu|.
\end{eqnarray}
Letting $a = W_{ij}^2 + W_{ii}W_{jj}$, $b = S_{ij} - W_{ij} + \vct{w}_i^T D \vct{w}_j$, and $c = X_{ij} + D_{ij}$ the minimum is achieved for:
\begin{eqnarray} \label{eqn:one-update}
  \mu = -c + \mathcal{S}(c - b/a, \lambda_{ij}/a),
\end{eqnarray}
where
\begin{eqnarray} \label{soft}
  \Soft(z,r) = \sign(z)\max\{|z| - r,0\}
\end{eqnarray}
is the soft-thresholding function. Similarly, when $i=j$, 
for $D'=D+\mu \vct{e}_i\vct{e}_i^T$, we get 
\begin{eqnarray}
  \tr(WD'WD') = \tr(WDWD) & + & 2\mu\vct{w}_i^T D\vct{w}_i + \mu^2 (W_{ii}^2).
\end{eqnarray}
Therefore the update rule for $D_{ii}$ can be computed by~\eqref{eqn:one-update} with $a=W_{ii}^2, b=S_{ii}-W_{ii}+\vct{w}_i^T D \vct{w}_i$, and $c=X_{ii}+D_{ii}$. 

Since $a$ and $c$ are easy to compute, the main computational cost arises 
while evaluating $\vct{w}_i^T D \vct{w}_j$, the third term contributing to 
coefficient $b$ above.  Direct computation requires $O(p^2)$ time.  Instead, 
we maintain $U = DW$ by updating two rows of the matrix $U$ for every variable 
update in $D$ costing $O(p)$ flops, and then compute $\vct{w}_i^T \vct{u}_j$ 
using $O(p)$ flops.  Another way to view this arrangement is that we maintain 
a decomposition $WDW = \sum_{k=1}^p \vct{w}_k \vct{u}_k^T$ throughout the 
process by storing the vectors $\vct{u}_k$, the columns of matrix $U$.  The 
representation allows us to compute $\vct{w}_i^T D \vct{w}_j$, the $(i, j)$ 
element of $WDW$, using only $O(p)$ flops, enabling fast 
computation of $\mu$ in~\eqref{eqn:one-update}.  In order to maintain the 
matrix $U$, we also need to update $2p$ elements, namely two coordinates of 
each $\vct{u}_k$ when $D_{ij}$ is modified.  We can compactly write the row 
updates of $U$ as follows: 
$\vct{u}_{i\cdot} \leftarrow \vct{u}_{i\cdot} + \mu \vct{w}_{j\cdot}$ and 
$\vct{u}_{j\cdot} \leftarrow \vct{u}_{j\cdot} + \mu \vct{w}_{i\cdot}$, where 
$\vct{u}_{i\cdot}$ refers to the $i$-th {\em row} vector of $U$.

\subsection*{Update rule when $X$ is diagonal}

The calculation of the Newton direction can be simplified if $X$ is a diagonal matrix. For example, this occurs in the first Newton iteration when we initialize \QUIC using the identity (or diagonal) matrix. When $X$ is diagonal, the Hessian $\nabla^2 g(X) = X^{-1}\otimes X^{-1}$ is a diagonal matrix, which indicates that all one variable sub-problems are independent of each other. Therefore, we only need to update each variable once to reach the optimum of~\eqref{eqn:subpb}. In particular, by examining~\eqref{eqn:one-update}, the optimal solution for $D'_{ij}$ is
\begin{equation}
  D'_{ij} = \begin{cases} \Soft\left(-\frac{S_{ij}}{W_{ii}W_{jj}}, 
    \frac{\lambda_{ij}}{W_{ii}W_{jj}}\right) & \text{ if } i\neq j,\\ 
    -X_{ii}+\Soft\left(X_{ii}-\frac{S_{ii}-W_{ii}}{W_{ii}^2}, 
    \frac{\lambda_{ii}}{W_{ii}^2}\right) & \text{ if }i=j, 
\end{cases}
  \label{eqn:close-form}
\end{equation}
where, as a reminder, $W_{ii} = 1/X_{ii}$. 
Thus, in this case, the
closed form solution for each variable can be computed in $O(1)$ time, so the time complexity for the first Newton direction is further reduced from $O(p^3)$ to $O(p^2)$. 

\subsection*{Updating only a subset of variables}

In our \QUIC algorithm we compute the Newton direction using only a subset of the variables we call the {\em free} set. We identify these variables in each Newton iteration based on the value of the gradient (we will discuss the details of the selection in Section~\ref{sec:sparsity}).  In the following, we define the Newton direction restricted to a subset $J$ of variables as the solution of a quadratic approximation.
\begin{definition} \label{def:DJ-define}
  Let $J$ denote a (symmetric) subset of variables. The Newton direction restricted to $J$ is defined as:
  \begin{equation} \label{eqn:dj-define}
    D^*_J(X) \equiv \arg\min_{\substack{D: D_{ij} = 0\\ \forall (i,j)\notin J}}
    \tr(\nabla g(X)^T D) +\frac{1}{2} \vect(D)^T \nabla^2 g(X)\vect(D) +
    \lnormg{X + D}.
  \end{equation}
\end{definition}
The cost to compute the Newton direction is thus substantially reduced when the free set is small, which as we will show in Section~\ref{sec:sparsity}, occurs when the optimal solution of the $\ell_1$-regularized Gaussian MLE is sparse. 

\subsection{Computing the Step Size} \label{sec:step-size}

Following the computation of the Newton direction $D^* = D^*_J(X)$ (restricted to 
the subset of variables $J$), we need to find a step size 
$\alpha\in\left(0,1\right]$ that ensures positive definiteness of the next 
iterate $X + \alpha D^*$ and leads to a sufficient decrease of the objective 
function.

We adopt Armijo's rule~(\cite{Bertsekas95},\cite{PT07a}) and try step-sizes 
$\alpha \in \{\beta^0, \beta^1, \beta^2, \dots\}$ with a constant decrease 
rate $0< \beta < 1$ (typically $\beta = 0.5$), until we find the smallest 
$k \in \mathbb{N}$ with $\alpha = \beta^k$ such that $X + \alpha D^*$ is (a) 
positive-definite, and (b) satisfies the following sufficient decrease 
condition:
\begin{eqnarray} \label{eqn:armijo}
  f(X + \alpha D^*) \leq f(X) + \alpha\sigma \delta, \ \ \delta = \tr(\nabla
  g(X)^T D^*) + \lnormg{X + D^*} - \lnormg{X},
\end{eqnarray}
where $0<\sigma<0.5$.  We verify positive definiteness while we compute the 
Cholesky factorization (costs $O(p^3)$ flops) needed for the objective 
function evaluation that computes $\log\det(X + \alpha D^*)$. This step 
dominates the computational cost in the step-size computations. We use the 
standard convention in convex analysis that  
$f(X)=+\infty$ when $X$ is not in the effective 
domain of $f$, i.e., $X$ is not positive definite. Using this 
convention, \eqref{eqn:armijo} enforces positive definiteness of 
$X+\alpha D^*$.
Next, we prove three important properties (P1--P3) that the line search procedure governed by~\eqref{eqn:armijo} satisfies:
\begin{enumerate}
  \item[P1.] The condition~\eqref{eqn:armijo} is satisfied for some 
    (sufficiently small) $\alpha$, establishing that the algorithm does not 
    enter into an infinite line search step.  We note that in 
    Proposition~\ref{lem:line-search} below we show that the line search 
    condition~\eqref{eqn:armijo} can be satisfied for any symmetric matrix $D$ (even one which is not the Newton direction). 
  \item[P2.] For the Newton direction $D^*$, the quantity $\delta$ in \eqref{eqn:armijo} is negative, which 
    ensures that the objective function decreases. Moreover, to guarantee that 
    $X_t$ converges to the global optimum, $\delta$ 
    should be small enough when the current iterate $X_t$ is far from the 
    optimal solution.  In Proposition~\ref{lem:Delta-bound2} we will prove 
    the stronger condition that $\delta\leq -(1/{M^2}) \|D\|_F^2$ for 
    some constant $M$. 
$\|D\|_F^2$ can be viewed as a measure of the distance from 
    optimality of the current iterate $X_t$, and this bound ensures that the 
    objective function decrease is proportional to $\|D\|_F^2$.
  \item[P3.] When $X$ is close enough to the global optimum, the step size 
    $\alpha=1$ will satisfy the line search condition~\eqref{eqn:armijo}. We 
    will show this property in Proposition~\ref{lm:alpha_1}.  Moreover, 
    combined with the global convergence of \QUIC proved in Theorem 
    \ref{thm:global_converge}, this property suggests that after a finite 
    number of iterations $\alpha$ will always be $1$; this also implies that 
    only one Cholesky factorization is needed per iteration (to evaluate $\log\det (X+\alpha D^*)$ for computing $f(X+\alpha D)$).
\end{enumerate}
We first prove a useful lemma. 
\begin{lemma} \label{lem:lnormg-triangle}
  For $X, D$ symmetric and $1\geq \alpha \geq 0$:
  \begin{equation*}
    \lnormg{X + \alpha D}  \leq
    \alpha\lnormg{X + D} + (1-\alpha)\lnormg{X}.
  \end{equation*}
\end{lemma}
\begin{proof}
Since $\lnormg{\cdot}$ is convex for $\Lambda \geq 0$, the following holds for $0\leq \alpha \leq 1$:
 \begin{equation} \label{eqn:lnormg-triangle}
   \lnormg{X + \alpha D} = \lnormg{\alpha(X + D) + (1-\alpha)X} \leq
   \alpha\lnormg{X + D} + (1-\alpha)\lnormg{X}.
 \end{equation}
\end{proof}
\begin{proposition}[corresponds to Property P1] \label{lem:line-search}
  For any $X\succ 0$ and symmetric $D$, there exists an $\bar{\alpha} > 0$ such that for all $\alpha < \bar{\alpha}$, the matrix $X + \alpha D$ satisfies the line search condition~\eqref{eqn:armijo}.
\end{proposition}
\begin{proof}
  When $\alpha < \sigma_n(X)/\|D\|_2$ (where $\sigma_n(X)$ stands for the smallest eigenvalue of $X$ and $\|D\|_2$ is the induced 2-norm of $D$, i.e., the largest eigenvalue of $D$), we have $\|\alpha D\|_2 < \sigma_n(X)$, which implies that $X+\alpha D \succ 0$.    
  So we can write:
  \begin{align*}
    f(X+\alpha D) - f(X) &= 
    g(X+\alpha D) - g(X) + \lnormg{X + \alpha D} - \lnormg{X} \\
    &\leq  g(X + \alpha D) - g(X) + \alpha(\lnormg{X + D} - \lnormg{X}), \ \ \text{ 
    (by Lemma \ref{lem:lnormg-triangle})} \\
    &= \alpha \tr((\nabla g(X))^T D) + O(\alpha^2) +  \alpha(\lnormg{X + D} - 
    \lnormg{X}) \\
    &= \alpha \delta + O(\alpha^2).  
  \end{align*}
  Therefore for any fixed $0 < \sigma < 1$ and sufficiently small $\alpha$, the line search condition \eqref{eqn:armijo} must hold.
\end{proof}
Before proving properties P2 and P3, we first state a few  
lemmas that will be useful in the sequel. 
\begin{lemma} \label{lem:Delta-0}
  $\delta = \delta_J(X)$ in the line search condition~\eqref{eqn:armijo} satisfies
  \begin{equation} \label{eqn:Delta-bound1}
    \delta = \tr((\nabla g(X))^T D^*) + \lnormg{X + D^*} - \lnormg{X} \leq -\vect(D^*)^T
    \nabla^2 g(X) \vect(D^*),
  \end{equation}
  where $D^*=D^*_{J}(X)$ is the minimizer of the $\ell_1$-regularized quadratic approximation defined in \eqref{eqn:dj-define}. 
\end{lemma}
\begin{proof}
 According to the definition of $D^*\equiv D^*_J(X)$ in~\eqref{eqn:dj-define}, for all $0 < \alpha < 1$ we have:
 \begin{eqnarray} \label{eqn:t1}
   \tr(\nabla g(X)^T D^*) + \frac{1}{2} \vect(D^*)^T \nabla^2 g(X) \vect(D^*) +
   \lnormg{X + D^*} \leq \nonumber \\
   \tr(\nabla g(X)^T \alpha D^*) + \frac{1}{2} \vect(\alpha D^*)^T \nabla^2
   g(X)\vect(\alpha D^*) + \lnormg{X + \alpha D^*}.
 \end{eqnarray}
 We combine~\eqref{eqn:t1} and Lemma~\ref{lem:lnormg-triangle} to yield:
 \begin{eqnarray*}
   \tr(\nabla g(X)^T D^*) + \frac{1}{2}\vect(D^*)^T \nabla^2 g(X) \vect(D^*) +
   \lnormg{X + D^*} \leq \nonumber \\
   \alpha \tr(\nabla g(X)^T D^*) + \frac{1}{2} \alpha^2 \vect(D^*)^T \nabla^2
   g(X) \vect(D^*) + \alpha\lnormg{X + D^*} + (1-\alpha)\lnormg{X}.
 \end{eqnarray*}
 Therefore
 \begin{equation*}
   (1-\alpha)[\tr(\nabla g(X)^T D^*) + \lnormg{X + D^*} - \lnormg{X}] +
   \frac{1}{2} (1 - \alpha^2) \vect(D^*)^T \nabla^2 g(X) \vect(D^*) \leq 0.
 \end{equation*}
 Divide both sides by $1-\alpha > 0$ to get:
 \begin{equation*}
   \tr(\nabla g(X)^T D^*) + \lnormg{X + D^*} - \lnormg{X}
   +\frac{1}{2}(1+\alpha)\vect(D^*)^T \nabla^2 g(X) \vect(D^*) \leq 0.
 \end{equation*}
 By taking the limit as $\alpha \uparrow 1$, we get: 
 \begin{equation*}
   \tr(\nabla g(X)^T D^*) + \lnormg{X + D^*} - \lnormg{X} \leq - \vect(D^*)^T
   \nabla^2 g(X) \vect(D^*),
 \end{equation*}
 which proves~\eqref{eqn:Delta-bound1}. 
 \end{proof}
Since $\nabla^2 g(X) = X^{-1} \otimes X^{-1}$ is positive definite, Lemma~\ref{lem:Delta-0} ensures that $\delta<0$ for all $X\succ 0$. Since the updates in our algorithm satisfy the line search condition~\eqref{eqn:armijo}, we have established that the function value is decreasing. It also follows that all the iterates $\{X_t\}_{t=0, 1, \dots}$ belong to the level set $U$ defined by:
\begin{equation} \label{eq:diff_level_set}
  U=\{X\mid f(X)\le f(X_0) \text{ and } X\in S^p_{+\!+}\}.  
\end{equation}
\begin{lemma} \label{lem:eig-bounds}
  The level set $U$ defined in~\eqref{eq:diff_level_set} is contained in the set $\{X\mid mI\preceq X\preceq MI\}$ for some constants $m, M > 0$, if we assume that the off-diagonal elements of $\Lambda$ and the diagonal elements of $S$ are positive.
\end{lemma}
\begin{proof}
  We begin the proof by showing that the largest eigenvalue of $X$ is less then $M$, a well chosen constant that depends only on $\Lambda$, $f(X_0)$ and the matrix $S$. We note that $S\succeq 0$ and $X\succ 0$ implies $\tr(SX)\geq 0$ and therefore:
  \begin{equation}
    f(X_0) > f(X) \geq -\log\det X + \lnormg{X}.
    \label{eqn:first-ineq}
  \end{equation}
  Since $\|X\|_2$ is the largest eigenvalue of the $p\times p$ matrix $X$, we have $\log\det X \leq p\log(\|X\|_2)$. Combine with~\eqref{eqn:first-ineq} and the fact that the off-diagonal elements of $\Lambda$ are no smaller than some $\lambda > 0$:
  \begin{equation}
    \lambda\sum_{_i\not=j} |X_{ij}| \leq \lnormg{X} \leq f(X_0) + p\log(||X||_2).
    \label{eqn:off-diag-X-ineq}
  \end{equation}
  Similarly, $\lnormg{X} \geq 0$ implies that:
  \begin{equation}
    \tr(SX) \leq f(X_0) + p\log(||X||_2).
    \label{eqn:SX-ineq}
  \end{equation}
  Next, we introduce $\alpha = \min_i S_{ii}$ and $\beta = \max_{i\not=j} |S_{ij}|$ and split $\tr(SX)$ into diagonal and off-diagonal terms in order to bound it:
  \begin{equation*}
    \tr(SX) = \sum_{i} S_{ii}X_{ii} + \sum_{i\not=j} S_{ij}X_{ij} \geq \alpha\tr(X) - \beta\sum_{i\not=j}|X_{ij}|.
  \end{equation*}
  Since $||X||_2 \leq \tr(X)$,
  \begin{equation*}
    \alpha ||X||_2 \leq \alpha\tr(X) \leq \tr(XS) + \beta\sum_{i\not=j}|X_{ij}|.
  \end{equation*}
  Combine with~\eqref{eqn:off-diag-X-ineq} and~\eqref{eqn:SX-ineq} to get:
  \begin{equation}\label{eqn:X-2-norm-ineq}
    \alpha ||X||_2 \leq (1 + \beta/\lambda)(f(X_0) + p\log(||X||_2)).
  \end{equation}
  The left hand side of inequality~\eqref{eqn:X-2-norm-ineq}, as a function of $||X||_2$, grows much faster than the right hand side (note $\alpha > 0$), and therefore $||X||_2$ can be bounded depending only on the values of $f(X_0)$, $\alpha$, $\beta$ and $\lambda$.

  In order to prove the lower bound, we consider the smallest eigenvalue of $X$ denoted by $a$ and use the upper bound on the other eigenvalues to get:
  \begin{eqnarray}
    f(X_0) > f(X) > -\log\det X \geq -\log a - (p-1)\log M,
  \end{eqnarray}
  which shows that $m = e^{-f(X_0)}M^{-(p-1)}$ is a lower bound for $a$. 
\end{proof}
We note that the conclusion of the lemma also holds if the conditions on 
$\Lambda$ and $S$ are replaced by only the requirement that the diagonal 
elements of $\Lambda$ are positive, see~\cite{Banerjee:glasso:2008}.  We 
emphasize that Lemma~\ref{lem:eig-bounds} allows the extension of the 
convergence results to the practically important case when the regularization 
does not penalize the diagonal. In subsequent arguments we will continue 
to refer to the minimum and maximum eigenvalues $m$ and $M$ established in 
the proof.
\begin{proposition}[corresponds to Property P2] \label{lem:Delta-bound2}
  $\delta = \delta_J(X)$ as defined in the line search condition~\eqref{eqn:armijo} satisfies
  \begin{equation} \label{eqn:Delta-bound2}
    \delta \leq -(1/\|X\|_2^2) \|D^*\|_F^2 \leq -(1/M^2)\|D^*\|_F^2, 
  \end{equation}
  where $M$ is as in Lemma~\ref{lem:eig-bounds}.
\end{proposition}
\begin{proof}
  We want to further bound the right hand side of~\eqref{eqn:Delta-bound1}. Since $\nabla^2 g(X) = X^{-1} \otimes X^{-1}$, the smallest eigenvalue of $\nabla^2 g(X)$ is $1/\|X\|^2_2$, and we combine with Lemma~\ref{lem:eig-bounds} to get~\eqref{eqn:Delta-bound2}.
\end{proof}
The eigenvalues of any iterate $X$ are bounded by Lemma~\ref{lem:eig-bounds}, and 
therefore $\nabla^2 g(X)=X^{-1}\otimes X^{-1}$ is Lipschitz continuous.
Next, we prove that $\alpha=1$ satisfies the line search condition when we are close to the global optimum.
\begin{proposition}[corresponds to Property P3]
  \label{lm:alpha_1}
  Assume that $\nabla^2 g$ is Lipschitz 
  continuous, i.e., $\exists L>0$ such that $\forall t>0$, 
  \begin{equation}
    \label{eqn:second_lip}
    \|\nabla^2 g(X+tD)-\nabla^2 g(X)\|_F \leq L\|tD\|_F = tL\|D\|_F. 
  \end{equation}
  Then, if 
  $X$ is close enough to $X^*$, the line search condition~\eqref{eqn:armijo} will be satisfied with step size $\alpha = 1$.
\end{proposition}
\begin{proof}
  We need to derive a bound for the decrease in the objective function value. We define $\tilde{g}(t) = g(X+tD)$, which yields $\tilde{g}''(t) = \vect(D)^T \nabla^2 g(X+tD)\vect(D)$.  
  First, we bound $|\tilde{g}''(t)-\tilde{g}''(0)|$: 
  \begin{eqnarray*}
  |\tilde{g}''(t)-\tilde{g}''(0)| & = &
  |\vect(D)^T (\nabla^2 g(X+tD)-\nabla^2 g(X))\vect(D)| \\
  & \leq & \|\vect(D)^T (\nabla^2 g(X+tD)-\nabla^2 g(X))\|_2 \|\vect(D)\|_2 \text{ (by Cauchy-Schwartz)} \\
  & \leq & \|\vect(D)\|_2^2   \|\nabla^2 g(X+tD)-\nabla^2 g(X)\|_2 \ \ \ \ \ \text{ (by definition of $\|\cdot\|_2$ norm)}\\
  & \leq & \|D\|_F^2 \|\nabla^2 g(X+tD)-\nabla^2 g(X)\|_F \ \text{ (since $\|\cdot\|_2\leq \|\cdot\|_f$ for any matrix)} \\
  & \leq & \|D\|_F^2 tL \|D\|_F \ \ \text{ (by \eqref{eqn:second_lip})} \\
  & = & tL\|D\|_F^3. 
\end{eqnarray*}
Therefore, an upper bound for $\tilde{g}''(t)$:
\begin{equation*}
  \tilde{g}''(t) \leq \tilde{g}''(0)+tL \|D\|_F^3
  = \vect(D)^T \nabla^2 g(X) \vect(D) + tL \|D\|_F^3. 
\end{equation*}
Integrate both sides to get
\begin{align*}
  \tilde{g}'(t) &\leq \tilde{g}'(0) + 
  t\vect(D)^T \nabla^2 g(X) \vect(D) + \frac{1}{2} t^2 L \|D\|_F^3 \\
  &=  \tr((\nabla g(X))^T D) + t \vect(D)^T \nabla^2 g(X) \vect(D)
  + \frac{1}{2} t^2 L \|D\|_F^3. 
\end{align*}
Integrate both sides again:
\begin{equation*}
  \tilde g(t) \leq \tilde{g} (0) + t \tr((\nabla g(X))^T D) 
  + \frac{1}{2} t^2 \vect(D)^T \nabla^2 g(X) \vect(D) + \frac{1}{6}
  t^3 L \|D\|_F^3. 
\end{equation*}
Taking $t=1$ we have 
\begin{align*}
  g(X+D)  \leq & \ g(X) + \tr(\nabla g(X)^T D) + \frac{1}{2} \vect(D)^T \nabla^2 g(X) 
  \vect(D) + \frac{1}{6}L \|D\|_F^3 \\
  f(X+D)  \leq & \  g(X)+\|X\|_{1,\Lambda} + (\tr(\nabla 
  g(X)^T D) + \|X+D\|_{1,\Lambda} -  \|X\|_{1,\Lambda}) \\
  & + \frac{1}{2}  \vect(D)^T \nabla^2 g(X) 
  \vect(D)  + \frac{1}{6}L \|D\|_F^3 \\ 
  \leq & f(X) + \delta  + \frac{1}{2} \vect(D)^T \nabla^2 g(X) 
  \vect(D) + \frac{1}{6}L \|D\|_F^3 \\
  \leq & f(X)+\frac{\delta}{2} + \frac{1}{6} L\|D\|_F^3  \text{ (by Lemma 
  \ref{lem:Delta-0})} \\
   \leq  & f(X) + (\frac{1}{2}-\frac{1}{6}LM^2 \|D\|_F)\delta \text{ (by Proposition 
  \ref{lem:Delta-bound2})}. 
\end{align*}
When $X$ is close to $X^*$, $D$ is close to $0$; therefore when 
$k$ is large enough, the second term must be smaller than $\sigma\delta$ (
$0<\sigma<0.5$), which implies that the line search 
condition~\eqref{eqn:armijo} holds with $\alpha = 1$.
\end{proof}

\subsection{Identifying which variables to update} \label{sec:sparsity}

In this section, we 
use the stationary condition of the Gaussian MLE problem to select a subset of  
variables to update in any Newton direction computation.  
Specifically, we partition the variables into {\em free} and {\em fixed} sets based on the value of the gradient 
at the start of the outer loop that computes the Newton direction. 
We define the {\em free} set $S_{free}$ and {\em fixed} set $S_{fixed}$ as:
\begin{align}
  X_{ij} 
    &\in S_{fixed} \ \ \text{if} \ \  |\nabla_{ij} g(X)| \leq \lambda_{ij} , \text{ 
    and }
    X_{ij}=0, \nonumber \\
    X_{ij} &\in S_{free}  \ \ \text{otherwise} \label{eqn:def_fixed}. 
\end{align}
We will now show that 
a Newton update restricted to the {\em fixed set} of variables would not 
change any of the coordinates in that set. In brief, the gradient condition 
$|\nabla_{ij}g(X)| < \lambda_{ij}$ entails that the inner coordinate descent 
steps, according to the update in \eqref{eqn:one-update},  would set these coordinates to zero, 
so they would not change since they were zero to begin with. 

At non-differentiable points of $\|X\|_1$, only sub-gradient can be defined. 
To derive the optimality condition, we begin by introducing the minimum-norm subgradient for $f$ and relate it to the optimal solution $X^*$ of~\eqref{eqn:primal_general}.
\begin{definition} \label{def:subgradient}
  We define the minimum-norm subgradient $\grad^S_{ij} f(X)$ as follows:
  \begin{equation*}
    \grad^S_{ij} f(X)=\begin{cases}
    \nabla_{ij} g(X)+\lambda_{ij} &\text{ if  } X_{ij}>0, \\
    \nabla_{ij} g(X)-\lambda_{ij} &\text{ if  } X_{ij}<0, \\
    \sign(\nabla_{ij}g(X))\max(|\nabla_{ij}g(X)| - \lambda_{ij},0) &\text{ if  
    } X_{ij} = 0. 
    \end{cases}
  \end{equation*}
\end{definition}

\begin{lemma} \label{lem:optimum-grad-partial}
  For any index set $J$, $\text{grad}^S_{ij}f(X)=0 \ \forall (i,j) \in J$ if and only if $\Delta^*=0$ is a solution of the following optimization problem:
  \begin{equation}
    \arg\min_{\Delta} f(X+\Delta) \text{ such that } \Delta_{ij}=0 \ \ \forall 
    (i,j)\notin J. 
    \label{eqn:subpb-opt}
  \end{equation}
\end{lemma}
\begin{proof}
  Any optimal solution $\Delta^*$ for \eqref{eqn:subpb-opt} must satisfy the following, 
 for all $(i, j)\in J$, 
  \begin{equation}
    \nabla_{ij} g(X+\Delta^*)  \begin{cases}
      = -\lambda_{ij} &\text{ if } X_{ij}>0, \\
     = \lambda_{ij} & \text{ if } X_{ij}<0, \\
     \in [-\lambda_{ij} \ \ \lambda_{ij}] &\text{ if } X_{ij}=0. 
    \end{cases}
    \label{eqn:opt-cond}
  \end{equation}
  It can be seen immediately that $\Delta^*=0$ satisfies~\eqref{eqn:opt-cond} if and only if $\grad^S_{ij} f(X)=0$ for all $(i,j) \in J$.
\end{proof}
In our case, $\nabla g(X)=S-X^{-1}$ and therefore
\begin{equation*}
  \text{grad}^S_{ij} f(X) = \begin{cases}
    (S-X^{-1})_{ij} + \lambda_{ij} &\text{ if } X_{ij}>0, \\
    (S-X^{-1})_{ij} - \lambda_{ij} &\text{ if } X_{ij}<0, \\
    \sign((S-X^{-1})_{ij}) \max(| (S-X^{-1})_{ij} | - \lambda_{ij},0) &\text{ if
    } X_{ij}=0. 
  \end{cases}
\end{equation*}
The definition of the minimum-norm sub-gradient is closely related to the 
definition of the {\em fixed} and {\em free} sets. A variable $X_{ij}$ belongs 
to the {\em fixed} set if and only if $X_{ij}=0$ and 
$\text{grad}^S_{ij} f(X) = 0$.  Therefore, taking $J=S_{fixed}$ in 
Lemma~\ref{lem:optimum-grad-partial} we arrive at the following crucial 
property of the fixed set.
\begin{proposition} \label{lem:fixed-free}
  For any $X_t$ and corresponding fixed and free sets $S_{fixed}$ and $S_{free}$ as defined by~\eqref{eqn:def_fixed}, $\Delta^*=0$ is the solution of the following optimization problem:
  \begin{equation*}
    \arg\min_{\Delta} f(X_t+\Delta) \text{ such 
      that } \Delta_{ij}=0 \ \ \forall (i,j)\in S_{free}.  
  \end{equation*}
\end{proposition}
Based on the above proposition, if we perform block coordinate descent 
restricted to the fixed set, then no updates would occur.  We then perform
the inner loop coordinate descent updates restricted to only the free set to 
find the Newton direction. With this modification, the number of variables over which 
we perform the coordinate descent update of~\eqref{eqn:one-update} can be potentially reduced 
from $p^2$ to the number of non-zeros in $X_t$. When the solution is sparse 
(depending on the value of $\Lambda$) the number of free variables is much 
smaller than $p^2$ and we can obtain huge computational gains as a result. 
In essence, we very efficiently select a subset of the coordinates that need 
to be updated.

The attractive facet of this modification is that it leverages the sparsity of the solution and intermediate iterates in a manner that falls within the block coordinate descent framework of~\cite{PT07a}. The index sets $J_1,J_2,\ldots$ corresponding to the block coordinate descent steps in the general setting of~\cite{PT07a}[p. 392] need to satisfy a Gauss-Seidel type of condition:
\begin{equation}
  \bigcup_{j=0,\ldots,T-1} J_{t+j} \supseteq \N \ \ \forall t = 1,2,\dots
  \label{eqn:set_assumption}
\end{equation}
for some fixed $T$, where $\N$ denotes the full index set. In our framework 
$J_1, J_3, \ldots$ denote the fixed sets at various iterations, and $J_2, J_4, \ldots$ denote the 
free sets. Since $J_{2i+1}$ and $J_{2i+2}$ is a partitioning of $\N$ the 
choice $T = 3$ will suffice. But will the size of the free set be small?  We 
initialize $X_0$ to a diagonal matrix, which is sparse. The following 
lemma shows that after a \emph{finite} number of iterations, the iterates 
$X_t$ will have a similar sparsity pattern as the limit $X^*$.  Lemma 
\ref{lem:shrink} is actually an immediate consequence of Lemma~\ref{lm:divide} 
in Section~\ref{sec:converge}.
\begin{lemma} \label{lem:shrink}
  Assume that $\{X_t\}$ converges to $X^*$, the optimal solution of~\eqref{eqn:primal_general}.  If for some index pair $(i,j)$, $|\nabla_{ij}g(X^*)| < \lambda_{ij}$ (so that $X^*_{ij} = 0$), then there exists a constant $\bar{t} > 0$ such that for all $t > \bar{t}$, the iterates $X_t$ satisfy
  \begin{equation} \label{eqn:shrink}
    |\nabla_{ij}g(X_t)| \leq \lambda_{ij} \ \text{ and } (X_t)_{ij}=0.
  \end{equation}
\end{lemma}
Note that $|\nabla_{ij}g(X^*)| < \lambda_{ij}$ implies $X^*_{ij} = 0$ from 
from the optimality condition of~\eqref{eqn:primal_general}.  A similar (so 
called shrinking) strategy is used in SVM and $\ell_1$-regularized logistic 
regression problems as mentioned in~\cite{GXY09a}.  In our experiments, we 
demonstrate that this strategy reduces the size of the free set very quickly.

\subsection{The block-diagonal structure of $X^*$} \label{sec:block-diagonal}

It has been shown recently by (\cite{hastie:2012},\cite{friedman:2011}) that when the thresholded covariance matrix $E$ defined by $E_{ij} = \Soft(S_{ij}, \lambda)= \sign(S_{ij})\max(|S_{ij}|-\lambda,0)$ has the following block-diagonal structure:\begin{equation}
  E=\left[
    \begin{matrix}
      E_1 & 0 & \dots & 0 \\
      0 & E_2 & \dots & 0 \\
      \vdots & \vdots & \vdots & \vdots \\
      0 & 0 & 0 & E_k 
    \end{matrix}
    \right],
  \label{eqn:block_diagonal_E}
\end{equation}
then the solution $X^*$ of the inverse covariance estimation problem \eqref{eqn:primal} also has 
the same block-diagonal structure:
\begin{equation*}
  X^*=\left[
    \begin{matrix}
      X_1^* & 0 & \dots & 0 \\
      0 & X_2^* & \dots & 0 \\
      \vdots & \vdots & \vdots & \vdots \\
      0 & 0 & 0 & X_k^*
    \end{matrix}
    \right]. 
\end{equation*}
This result can be extended to the case when the elements are penalized differently, i.e., $\lambda_{ij}$'s are different. 
Then, if $E_{ij}=\Soft(S_{ij}, \lambda_{ij})$ is block diagonal, so is the solution $X^*$ of \eqref{eqn:primal_general}, see \cite{dcsvm}. 
Thus each $X_i^*$ can be computed independently. Based on this observation one can decompose the problem into sub-problems of smaller sizes. Since the complexity of solving \eqref{eqn:primal_general} is $O(p^3)$, solving several smaller sub-problems is much faster. In the following, we show that our updating rule and {\em fixed}/{\em free} set selection technique can automatically detect this block-diagonal structure for free.

Recall that we have a closed form solution in the first iteration when the input is a diagonal matrix. 
Based on \eqref{eqn:close-form}, since $X_{ij}=0$ for all $i\neq j$ in this step, we have
\begin{equation*}
  D_{ij} = X_{ii}X_{jj}\Soft(-S_{ij}, \lambda_{ij}) =
  -X_{ii}X_{jj}\Soft(S_{ij}, \lambda_{ij}) \ \ \text{ for all } i\neq j.
\end{equation*}
We see that after the first iteration the nonzero pattern of $X$ will be exactly the same as the nonzero pattern of the thresholded covariance matrix $E$ as depicted in~\eqref{eqn:block_diagonal_E}. In order to establish that the same is true at each subsequent step, we complete our argument using induction, by showing that the structure is preserved.

More precisely, we show that the off-diagonal blocks always belong to the {\em fixed} set if $|S_{ij}|\leq\lambda_{ij}$.  Recall the definition of the {\em fixed} set in \eqref{eqn:def_fixed}. We need to check whether $|\nabla_{ij} g(X)|\leq\lambda_{ij}$ for all $(i,j)$ in the off-diagonal blocks of $E$. Taking the inverse preserves the diagonal structure, and therefore $\nabla_{ij} g(X) = S_{ij}-X^{-1}_{ij}=S_{ij}$. We conclude noting that $E_{ij}=0$ implies that $|\nabla_{ij}g(X)|\leq\lambda_{ij}$, meaning that $(i,j)$ will belong to the {\em fixed} set.

We decompose the matrix into smaller blocks prior to running Cholesky 
factorization to avoid the $O(p^3)$ time complexity.  The connected components 
of $X$ can be detected in $O(\|X\|_0)$ time, which is very efficient when $X$ is sparse. 
The detailed description of \QUIC is presented in 
Algorithm~\ref{alg:quadratic}.
\begin{algorithm}[htp!]
  \DontPrintSemicolon
  \caption{QUadratic approximation method for sparse Inverse Covariance
    learning ({\em QUIC}) \label{alg:quadratic}}
  \SetKwInOut{Input}{Input}

  \SetKwInOut{Output}{Output} \Input{Empirical covariance matrix $S$ (positive semi-definite $p\times p$), regularization parameter matrix $\Lambda$, initial $X_0$, inner stopping tolerance $\epsilon$, parameters $0<\sigma<0.5, \ 0<\beta<1$}
  \Output{Sequence of $X_t$ converging to $\arg\min_{X \succ 0} f(X)$, where $f(X) = g(X) + h(X)$, where $g(X)=-\log\det X + \text{tr}(SX), h(X)=\lnormg{X}$.}
  Compute $W_0 = X_0^{-1}$.\;
  \For{$t = 0, 1, \ldots$} {
    $D = 0$, $U = 0$ \;
    \While{not converged}{
      Partition the variables into fixed and free sets: \;
      $S_{fixed} := \{(i,j) \mid |\nabla_{ij}g(X_t)|<\lambda_{ij} \text{ and } (X_t)_{ij}=0\}$, 
      $S_{free} := \{(i,j) \mid |\nabla_{ij} g(X_t)|\geq \lambda_{ij} \text{ or } (X_t)_{ij}\neq 0 \}$. \;
      \For{$(i, j) \in S_{free}$}{
        $a = w_{ij}^2 + w_{ii}w_{jj}$ \;
        $b = s_{ij} - w_{ij} + \vct{w}_{\cdot i}^T \vct{u}_{\cdot j}$ \;
        $c = x_{ij} + d_{ij}$ \;
        $\mu = -c + \Soft(c - b/a, \lambda_{ij}/a)$ \;
        $d_{ij} \leftarrow d_{ij} + \mu$ \;
        $\vct{u}_{i\cdot} \leftarrow \vct{u}_{i\cdot} + \mu \vct{w}_{j \cdot}$ \;
        $\vct{u}_{j\cdot} \leftarrow \vct{u}_{j\cdot} + \mu \vct{w}_{i \cdot}$ \;
      }
    }
    \For{$\alpha = 1, \beta, \beta^2, \dots$}{ \label{alg:line-search}
      Compute the Cholesky factorization $LL^T = X_t + \alpha D$. \;
      \If{$X_t + \alpha D \not\succ 0$}{
        continue
      }
      Compute $f(X_t + \alpha D)$ from $L$ and $X_t + \alpha D$ \;
      \If{$f(X_t + \alpha D) \leq f(X_t) + \alpha\sigma \left[\tr(\nabla
          g(X_t) D) + \lnormg{X_t + D} - \lnormg{X}\right]$}
         {break}
    }
    $X_{t+1} = X_t + \alpha D$ \;
    Compute $W_{t+1} = X_{t+1}^{-1}$ reusing the Cholesky factor. \;
  }
\end{algorithm}

\section{Convergence Analysis} \label{sec:converge}

In Section \ref{sec:quadratic-solver}, we introduced the main ideas behind our 
\QUIC algorithm. In this section, we first prove that \QUIC converges to the 
global optimum, and then show that the convergence rate is quadratic.  
\cite{Banerjee:glasso:2008} showed that for the special case where 
$\Lambda_{ij}=\lambda$ the optimization problem~\eqref{eqn:primal} has a 
unique global optimum and that the eigenvalues of the primal optimal solution 
$X^*$ are bound. In the following theorem we show this result for 
more general $\Lambda$ where only the off-diagonal elements need to be 
positive. 

\begin{theorem} \label{lem:unique}
  There exists a unique minimizer $X^*$ for the optimization problem~\eqref{eqn:primal_general}.
\end{theorem}
\begin{proof}
  According to Lemma~\ref{lem:eig-bounds}, the level set $U$ defined in 
  \eqref{eq:diff_level_set} contains all the iterates, and it is in turn contained 
  in the compact set $S \equiv \{X\mid mI\preceq X\preceq MI\}$. According to 
  the Weierstrass extreme value theorem \citep{TMA74a}, any continuous 
  function in a compact set attains its minimum. Furthermore, 
  $\nabla^2 g(X) = X^{-1}\otimes X^{-1}$ implies 
  $\nabla^2 g(X) \succeq M^{-2}I$. 
  Since $\|X\|_{1,\Lambda}$ is convex and $-\log\det(X)$ is strongly convex, 
  we have that $f(X)$ is strongly convex on the compact set $S$, and therefore the 
  minimizer $X^*$ is unique \citep{TMA74a}.
\end{proof}

\subsection{Convergence Guarantee} \label{sec:conv-guarantee}

In order to show that \QUIC converges to the optimal solution, we consider a more general setting of the quadratic approximation algorithm: at each iteration, the iterate $Y_t$ is updated by $Y_t = Y_t + \alpha_t D_{J_t}(Y_t)$ where $J_t$ is a subset of variables chosen to update at iteration $t$, $D_{J_t}(Y_t)$ is the Newton direction restricted to $J_t$ defined by \eqref{eqn:dj-define}, and $\alpha_t$ is the step size selected by the Armijo rule mentioned in Section \ref{sec:step-size}. The algorithm is summarized in Algorithm \ref{alg:QP}. Similar to the block coordinate descent framework of \cite{PT07a}, we assume the index set $J_t$ satisfies a Gauss-Seidel type of condition:
\begin{equation}
  \bigcup_{j=0,\ldots,T-1} J_{t+j} \supseteq \N \ \ \forall t = 1,2,\dots. 
  \label{eqn:gauss-seidel}
\end{equation}

\begin{algorithm}[htp!]
  \DontPrintSemicolon
  \caption{General Block Quadratic Approximation method for Sparse Inverse Covariance Learning \label{alg:QP}}
  \SetKwInOut{Input}{Input}\SetKwInOut{Output}{Output}
  \Input{Empirical covariance matrix $S$ (positive semi-definite $p\times p$), regularization parameter matrix $\Lambda$, initial $Y_0$, inner stopping tolerance $\epsilon$}
  \Output{Sequence of $Y_t$.}
  \For{$t = 0,1,\ldots$}{
    Generate a variable subset $J_t$. \;
    Compute the Newton direction $D_t\equiv D_{J_t}(Y_t)$ by 
    \eqref{eqn:dj-define}. \; 
    Compute the step-size $\alpha_t$ using an {\em Armijo}-rule based step-size selection in \eqref{eqn:armijo}. \;
    Update $Y_{t+1} = Y_t + \alpha_t D_t$. \;
  }
\end{algorithm}

In \QUIC, $J_1, J_3, \dots$ denote the fixed sets and $J_2, J_4, \dots$ denote 
the free sets. If $\{X_t\}_{t=1,2,\dots}$ denotes the sequence generated by 
\QUIC, then
\begin{equation*}
  Y_1=Y_2 = X_1, Y_3=Y_4= X_2, \dots, Y_{2i-1}=Y_{2i}=X_{i}. 
\end{equation*}
Moreover, since each $J_{2i-1}$ and $J_{2i}$ is a partitioning of $\N$, the 
choice $T=3$ will satisfy \eqref{eqn:gauss-seidel}. In the rest of this 
section, we will show that $\{Y_t\}_{t=1,2,\dots}$ converges to the global 
optimum, thus $\{X_t\}_{t=1,2,\dots}$ generated by \QUIC also converges to the 
global optimum.

Our first step towards the convergence proof is a lemma on convergent 
subsequences of $X_t$. 
\begin{lemma} \label{lem:subsequence}
  For any convergent subsequence $Y_{s_t} \rightarrow Y^*$, we have $D_{s_t} \equiv D_{J_{s_t}}(Y_{s_t}) \rightarrow 0$.
\end{lemma}
\begin{proof}
  The objective value decreases according to the line search condition~\eqref{eqn:armijo} and Proposition \ref{lem:Delta-bound2}. According to Lemma~\ref{lem:eig-bounds}, $f(Y_{s_t})$ cannot converge to negative infinity, so $f(X_{s_t})-f(X_{s_{t+1}}) \rightarrow 0$. The line search condition~\eqref{eqn:armijo} implies that we have $\alpha_{s_t} \delta_{s_t} \rightarrow 0$.
  
  We proceed to prove the statement by contradiction.  If $D_{s_t}$ does not converge to $0$, then there exist an infinite index set ${\cal T}\subseteq \{s_1, s_2, \ldots\}$ and $\eta > 0$ such that $\|D_t\|_F > \eta$ for all $t \in {\cal T}$.  According to Proposition~\ref{lem:Delta-bound2}, $\delta_{s_t}$ is bounded away from $0$, therefore $\delta_{s_t}\not\rightarrow 0$, while $\alpha_{s_t}\rightarrow 0$.  We can assume without loss of generality that $\alpha_{s_t} < 1$, that is the line search condition is not satisfied in the first attempt. We will work in this index set ${\cal T}$ in the derivations that follow.
  
  The line search step size $\alpha_t < 1$ ($t\in{\cal T})$ satisfies~\eqref{eqn:armijo}, but $\overline{\alpha}_t = \alpha_t/\beta$ does not satisfy~\eqref{eqn:armijo} by the minimality of our line search procedure. So we have:
  \begin{equation} \label{eqn:proof_ine}
    f(Y_t + \overline{\alpha}_tD_t) - f(Y_t) \geq \sigma
    \overline{\alpha}_t\delta_t.
  \end{equation}
  If $Y_t + \overline{\alpha}_t D_t$ is not positive definite, then we define $f(Y_t + \overline{\alpha}_t D_t)$ to be $\infty$, so~\eqref{eqn:proof_ine} still holds.  We expand the definition of $f$ and apply Lemma~\ref{lem:lnormg-triangle}:
  \begin{align*}
    \sigma \overline{\alpha_t} \Delta_t 
    &\leq g(Y_t + \overline{\alpha}_t D_t) - g(Y_t) + \lnormg{Y_t +
      \overline{\alpha}_t D_t} - \lnormg{Y_t} \\
    &\leq g(Y_t + \overline{\alpha}_t  D_t) - g(Y_t) +
      \overline{\alpha}_t\lnormg{|Y_t+D_t} + (1-\overline{\alpha}_t)
      \lnormg{Y_t} - \lnormg{Y_t} \\
    &= g(Y_t+\overline{\alpha}_t D_t) -
      g(Y_t) + \overline{\alpha}_t(\lnormg{Y_t+D_t} - \lnormg{Y_t}),
    \forall t\in {\cal T}.
  \end{align*}
  By the definition of $\delta_t$ we have: 
  \begin{align*}
    &\sigma \delta_t \leq \frac{g(Y_t + \overline{\alpha}_t D_t) -
      g(Y_t)}{\overline{\alpha}_t  } + \delta_t - \tr(\nabla g(Y_t)^T D_t), \\
    &(1-\sigma)(-\delta_t) \leq \frac{g(Y_t+\overline{\alpha}_t D_t) -
      g(Y_t)}{\overline{\alpha}_t  }-\tr(\nabla g(Y_t)^T D_t).
  \end{align*}
By Proposition~\ref{lem:Delta-bound2} we have:
  \begin{align*}
    &(1-\sigma)M^{-2}\|D_t\|^2_F \leq
    \frac{g(Y_t+\overline{\alpha}_t D_t)-g(Y_t) }{\overline{\alpha}_t  }
    - \tr(\nabla g(Y_t)^T D_t) \\
    & (1-\sigma)M^{-2}\|D_t\|_F \leq \frac{g(Y_t + \overline{\alpha}_t \|D_t\|_F
      \frac{D_t}{\|D_t\|_F}) - g(Y_t)}{\overline{\alpha}_t\|D_t\|_F} - \tr(\nabla
    g(Y_t)^T \frac{D_t}{\|D_t\|_F}).
  \end{align*}
  We set $\hat{\alpha}_t = \overline{\alpha}_t \|D_t\|_F $. Since $\|D_t\|_F>\eta$ for all $t\in {\cal T}$ we have:
  \begin{align} 
    (1-\sigma)M^{-2} \eta &\leq \frac{g(Y_t+\hat{\alpha}_t
      \frac{D_t}{\|D_t\|_F})
      - g(Y_t)}{\hat{\alpha}_t}- \tr(\nabla
    g(Y_t)^T \frac{D_t}{\|D_t\|_F}) \nonumber \\
      & = \frac{\hat{\alpha}_t\tr(\nabla g(Y_t) \frac{D_t}{\|D_t\|_F}) + 
      O(\hat{\alpha}_t^2)}{\hat{\alpha}_t} -  
 \tr(\nabla
    g(Y_t)^T \frac{D_t}{\|D_t\|_F})\nonumber \\
      &= O(\hat{\alpha}_t) \label{eqn:final1}
  \end{align}
  Again, by Proposition~\ref{lem:Delta-bound2},
  \begin{equation*}
    -\alpha_t \delta_t\geq \alpha_t M^{-2} \|D_t\|^2_F \geq M^{-2} \alpha_t  
    \|D_t\|_F \eta. 
  \end{equation*}
 Since $\{\alpha_t\delta_t\}_t \rightarrow 0 $, it follows that $\{\alpha_t\|D_t\|_F\}_t\rightarrow 0$ and $\{\hat{\alpha}^k\}_t\rightarrow 0$. Taking limit of~\eqref{eqn:final1} as $t\in \bar{\cal T}$ and $t\rightarrow \infty$, we have
  \begin{equation*}
    (1-\sigma)M^{-2}\eta \leq 0, 
  \end{equation*}
  a contradiction, finishing the proof.
\end{proof}

In Lemma \ref{lem:subsequence}, we prove that $D_{J_t}$ converges to zero for 
the converging subsequence. Next we show that $D_J$ is closely related to 
$\text{grad}^S f(Y)$ defined in Definition \ref{def:subgradient}, which in 
turn is an indicator of optimality as proved in Lemma \ref{lem:optimum-grad-partial}.

\begin{lemma} \label{lem:index-conv}
  For any index set $J$ and positive definite $Y$, $D_J(Y) = 0$ if and only if $\grad^S_{ij} f(Y) = 0$ for all $(i,j)\in J$.
\end{lemma}
\begin{proof}
  $D_J(Y) = 0$ if and only if $D = 0$ satisfies the optimality condition of~\eqref{eqn:dj-define}. The condition can be written as~\eqref{eqn:opt-cond} restricted to $(i,j)\in J$, which in turn is equivalent to the optimality condition of $f$.  Therefore $D_J(Y)=0$  iff $\grad^S_{ij} f(Y) = 0$ for all $(i,j)\in J$.
\end{proof}
Based on these lemmas, we are now able to prove our main convergence theorem. 
\begin{theorem}
  Algorithm \ref{alg:QP} converges to a unique global optimum $Y^*$.
  \label{thm:global_converge}
\end{theorem}
\begin{proof}
  Assume a subsequence $\{Y_t\}_{\cal T}$ converges to $\bar{Y}$. Since the choice of the index set $J_t$ selected at each step is finite, we can further assume that $J_t = \bar{J}_0$ for all $t\in {\cal T}$, considering a subsequence of ${\cal T}$ if necessary.  From Lemma~\ref{lem:subsequence}, $D_{\bar{J}_0}(Y_t)\rightarrow 0$. By the continuity of $\nabla g(Y)$ and $\nabla^2 g(Y)$, it is easy to show $D_{\bar{J}_0}(Y_t)\rightarrow D_{\bar{J}_0}(\bar{Y})$. Therefore $D_{\bar{J}_0}(\bar{Y}) = 0$. Based on Lemma~\ref{lem:index-conv}, we have
  \begin{equation*}
    \text{grad}^S_{ij} f(Y) = 0 \ \ \text{ for all } (i,j) \in \bar{J}_0. 
  \end{equation*}
Furthermore, $\{D_{\bar{J}_0}(Y_t)\}_{\cal T}\rightarrow 0$ and $\|Y_t-Y_{t+1}\|_F\leq \|D_{\bar{J}_0}(Y_t)\|_F$, so $\{Y_{t+1}\}_{t\in {\cal T}}$ also converges to $\bar{Y}$. By considering a subsequence of ${\cal T}$ if necessary, we can further assume that $J_{t+1} = \bar{J}_1$ for all $t\in {\cal T}$. By the same argument, we can show that $\{D_{J_{t+1}}(Y_t)\}_{\cal T}\rightarrow 0$, so $D_{\bar{J}_1}(\bar{Y}) = 0$.  Similarly, we can show that $D_{\bar{J}_t}(\bar{Y})=0 \ \forall t=0,\dots,T-1$ can be assumed for an appropriate subset of ${\cal T}$. With assumption~\eqref{eqn:gauss-seidel} and Lemma~\ref{lem:index-conv} we have 
  \begin{equation}
    \grad^S_{ij} f(\bar{Y})=0 \ \forall i,j.
    \label{eqn:ppp}
  \end{equation}
  Using Lemma~\ref{lem:optimum-grad-partial} with $J$ is the set of all variables, we can show that \eqref{eqn:ppp} implies $\bar{Y}$ is the global optimum. 
\end{proof}

\subsection{Asymptotic Convergence Rate} \label{sec:convergence-rate}

\noindent {\bf Newton methods on constrained minimization problems: }

\noindent The convergence rate of the Newton method on bounded constrained minimization has been studied in~\cite{ESL66a} and~\cite{Dunn80a}. Here, we briefly mention their results.

Assume we want to solve a constrained minimization problem
\begin{equation*}
  \min_{\vct{x}\in \Omega} F(\vct{x}), 
\end{equation*}
where $\Omega$ is a nonempty subset of $R^n$ denoting the constraint set and $F:R^n\rightarrow R$ has a second derivative $\nabla^2 F(\vct{x})$. Then beginning from $\vct{x}_0$, 
the natural Newton updates entail computing the $(k+1)$th iterate $x_{k+1}$ as
\begin{equation}
  \vct{x}_{k+1}=\arg\min_{\vct{x}\in \Omega} \nabla F(\vct{x}_k)^T(\vct{x}-\vct{x}_k) + \frac{1}{2} (\vct{x}-\vct{x}_k)^T \nabla^2
  F(\vct{x}_k) (\vct{x}-\vct{x}_k). 
  \label{eqn:bound-newton}
\end{equation}
For simplicity, we assume $F$ is strictly convex, and has a unique minimizer $\vct{x}^*$ in $\Omega$. Then the following theorem holds
\begin{theorem}[From Theorem 3.1 in \cite{Dunn80a}]
  \label{thm:dunn}
  Assume $F$ is strictly convex, has a unique minimizer $\vct{x}^*$ in $\Omega$, 
  and that $\nabla^2 F(\vct{x})$ is Lipschitz continuous. Then for all 
  $\vct{x}_0$ sufficiently close to $\vct{x}^*$, the sequence $\{\vct{x}_k\}$ 
  generated by \eqref{eqn:bound-newton} converges quadratically to $\vct{x}^*$.
\end{theorem}
This theorem is proved in~\cite{Dunn80a}. In our case, the objective function 
$f(X)$ is non-smooth so that Theorem~\ref{thm:dunn} does not directly apply. 
Instead, we will 
first show that after a finite number of  
iterations the sign of the iterates $\{X_t\}$ generated by \QUIC will not change, 
so that we can then use Theorem~\ref{thm:dunn} to establish asymptotic quadratic convergence.
\\

\noindent {\bf Quadratic convergence rate for \QUIC: }

\noindent Unlike the previous section, our Algorithm \ref{alg:QP} does not perform an unrestricted Newton update: it iteratively selects variable subsets $\{J_t\}_{t=1,\dots}$ in the manner of {\em fixed} and {\em free} sets, and performs Newton directions restricted to the free sets. In the following, we show that the sequence $\{X_t\}_{t=1,2,\dots}$ generated by \QUIC does converge quadratically to the global optimum.

Assume $X^*$ is the optimal solution, then we can divide the index set with $\lambda_{ij}\neq 0$ into three subsets: 
\begin{align}
  P &= \{(i,j)\mid X^*_{ij}>0\}, \nonumber\\ 
  N &= \{(i,j)\mid X^*_{ij}<0\}, \label{eqn:set_def}\\
  Z &= \{(i,j)\mid X^*_{ij}=0\}. \nonumber 
\end{align}
From the optimality condition of $X^*$,
\begin{equation}
  \nabla_{ij} g(X^*) \begin{cases}
    = -\lambda_{ij} & \text{ if } (i,j)\in P,  \\
    = \lambda_{ij} & \text{ if } (i,j) \in N,  \\
    \in [-\lambda_{ij},\lambda_{ij}] & \text{ if } (i,j) \in Z. 
  \end{cases}
  \label{eqn:set_prop}
\end{equation}

\begin{lemma}
  Assume that the sequence $\{X_t\}$ converges to the global optimum $X^*$. 
  Then there exists a $\bar{t}$ such that
  \begin{equation}
    (X_t)_{ij} \ \begin{cases} \ 
      \geq 0 &\text{ if } (i,j)\in P \\
   \   \leq 0 &\text{ if } (i,j)\in N \\
   \   = 0 & \text{ if } (i,j)\in Z
    \end{cases}
    \label{eqn:assumption}
  \end{equation}
  for all $t > \bar{t}$. 
  \label{lm:divide}
\end{lemma}
\begin{proof}
  We prove the case for $(i,j)\in P$ by contradiction, the other two cases can be handled similarly. If we cannot find a $\bar{t}$ satisfying \eqref{eqn:assumption}, then there exists an infinite subsequence $\{X_{s_t}\}$ such that $(X_{s_t})_{ij} < 0$.  We consider the update from $X_{s_t - 1}$ to $X_{s_t}$.  From Lemma~\ref{lm:alpha_1}, we can assume that $s_t$ is large enough so that the step size equals $1$, therefore $X_{s_t}=X_{s_t - 1} + D(X_{s_t-1})$ where $D(X_{s_t-1})$ is defined in \eqref{eqn:dj-define}. 
Since $(X_{s_t})_{ij} = (X_{s_t - 1})_{ij} + (D(X_{s_t-1}))_{ij} < 0$, from the
optimality condition of~\eqref{eqn:dj-define} we have
\begin{equation}
  \big(\nabla g(X_{s_t - 1}) + \nabla^2 g(X_{s_t - 1})\vect(D(X_{s_t-1}))\big)_{ij}= 
  \lambda_{ij}. 
    \label{eqn:sub-opt}
  \end{equation}
  Since $D(X_{s_t-1})$ converges to $0$,~\eqref{eqn:sub-opt} implies that $\{\nabla_{ij} g(X_{s_t - 1})\}$ will converge to $\lambda_{ij}$.  However, \eqref{eqn:set_prop} implies $\nabla_{ij} g(X^*) = -\lambda_{ij}$, and by the continuity of $\nabla g$ we get that $\{\nabla_{ij}g(X_t)\}$ converges to $\nabla_{ij} g(X^*)=-\lambda_{ij}$, a contradiction, finishing the proof.
\end{proof}
The following lemma shows that the coordinates from the fixed set remain zero 
after a finite number of iterations. 
\begin{lemma}
  Assume $X_t\rightarrow X^*$.  There exists a $\bar{t} > 0$ such that variables in $P$ or $N$ will not be selected to be in the fixed set $S_{fixed}$, when $t > \bar{t}$. That is,
  \begin{equation*}
    S_{fixed} \subseteq Z. 
  \end{equation*}
  \label{lm:fixright}
\end{lemma}
\begin{proof} 
  Since $X_t$ converges to $X^*$, $(X_t)_{ij}$ converges to $X^*_{ij}>0$ if 
  $(i,j)\in P$ and to $X^*_{ij} < 0$ if $(i,j) \in N$. Recall that $(i,j)$ 
  belongs to the fixed set only if $(X_t)_{ij}=0$. When $t$ is large enough, 
  $(X_t)_{ij}\neq 0$ when $X_t \in P\cup N$, therefore $P$ and $N$ will be 
  disjoint from the fixed set. Moreover, by the definition of the fixed set 
  \eqref{eqn:def_fixed}, indexes with $\lambda_{ij}=0$ will never be selected. 
  We proved that the fixed set will be a subset of $Z$ when $t$ is large enough. 
\end{proof}

\begin{theorem}
  The sequence $\{X_t\}$ generated by the \QUIC algorithm converges quadratically to $X^*$.
\end{theorem}
\begin{proof}
  First, if the index sets $P,N$ and $Z$ (related to the optimal solution)  are given, the optimum of \eqref{eqn:primal} is the same as the optimum of the following constrained minimization problem: 
  \begin{align}
    \min_X \ & \ -\log\det(X)+\tr(SX)+\sum_{(i,j)\in P}\lambda_{ij} X_{ij}
    -\sum_{(i,j)\in N} \lambda_{ij} X_{ij} \notag\\
    \text{s.t. } \ & \ X_{ij}\geq 0 \ \ \forall (i,j)\in
    P, \label{eqn:newton_pb}\\
    & \ X_{ij}\leq 0 \ \ \forall (i,j)\in N, \notag\\
    & \ X_{ij}=0 \ \ \forall (i,j)\in Z.
    \notag
  \end{align}
  In the following, we show that when $t$ is large enough, \QUIC solves the minimization problem described by~\eqref{eqn:newton_pb}.
  \begin{enumerate}
    \item The constraints in \eqref{eqn:newton_pb} are satisfied by \QUIC iterates after a finite number of steps, as shown in Lemma~\ref{lm:divide}. Thus, the $\ell_1$-regularized Gaussian MLE~\eqref{eqn:primal_general} is equivalent to the smooth constrained objective \eqref{eqn:newton_pb}, since the constraints in \eqref{eqn:newton_pb} are satisfied when solving \eqref{eqn:primal_general}.
    \item Since the optimization problem in \eqref{eqn:newton_pb} is smooth, 
      it can be solved using constrained Newton updates as in 
      \eqref{eqn:bound-newton}. The  \QUIC update direction $D_J(X_t)$ is 
      restricted to a set of free variables in $J$. This is exactly equal 
      to the unrestricted Newton update as in \eqref{eqn:bound-newton}, after 
      a finite number of steps, as established by Lemma~\ref{lm:fixright}. In 
      particular, at each iteration the fixed set is contained in $Z$, which 
      is 
      the set which always satisfies $(D_t)_Z=0$ for large enough $t$. 
    \item Moreover, by Lemma~\ref{lm:alpha_1} the step size is $\alpha=1$ when $t$ is 
      large enough.
  \end{enumerate}
   
Therefore our algorithm is equivalent to the constrained Newton method in~\eqref{eqn:bound-newton}, which in turn converges quadratically to the optimal solution of~\eqref{eqn:newton_pb}.  Since the revised problem~\eqref{eqn:newton_pb} and our original problem~\eqref{eqn:primal_general} has the same minimum, we have shown that \QUIC converges quadratically to the  optimum of~\eqref{eqn:primal_general}.
\end{proof}

In the next section, we show that this asymptotic convergence behavior of \QUIC is corroborated empirically as well.

\section{Experimental Results} \label{sec:experiments}

\subsection{Stopping condition for solving the sub-problems}
\label{sec:stopping}
In the convergence analysis of Section~\ref{sec:converge}, we assumed that 
each Newton direction $D_t$ is computed exactly by solving the Lasso 
subproblem~\eqref{eqn:dj-define}. In our implementation we use an iterative 
solver to compute $D_t$, which after a finite set of iterations only solves 
the problem to some accuracy. In the 
first experiment we explore how varying the accuracy to which we compute the 
Newton direction affects overall performance. In Figure~\ref{fig:stopping} we 
plot the total run times for the ER biology dataset 
from~\citep{LL10a} correspond to different numbers of inner iterations 
used in the coordinate descent solver of \QUIC.
\begin{figure}[htp!]
\begin{tabular}{cc}
  \subfigure[\QUIC with varying number of inner iterations (on ER dataset). \label{fig:stop_1} ]{\includegraphics[width=0.49\textwidth]{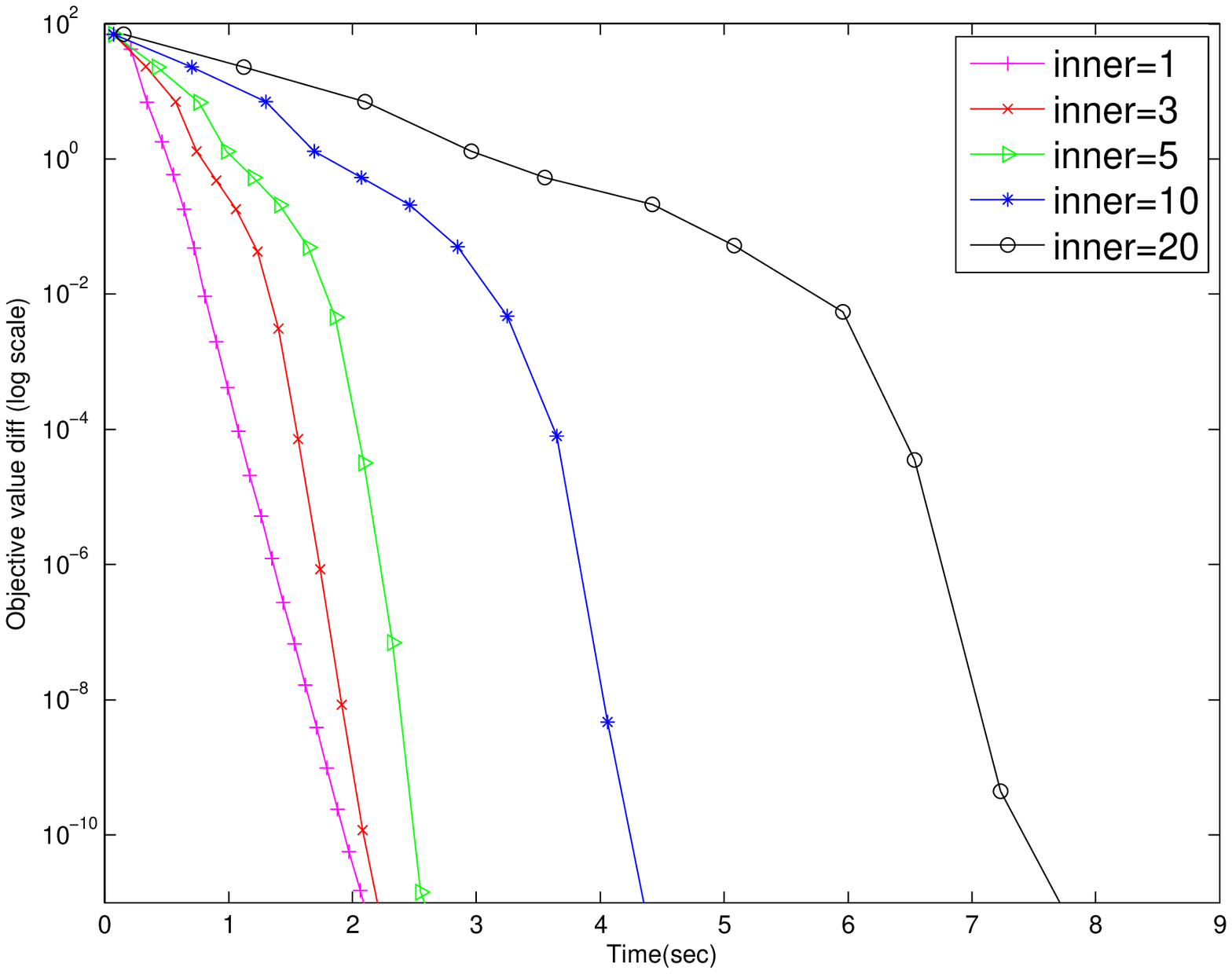}}
&
\subfigure[\QUIC with adaptive stopping condition for inner iterations. \label{fig:stop_2}]{\includegraphics[width=0.49\textwidth]{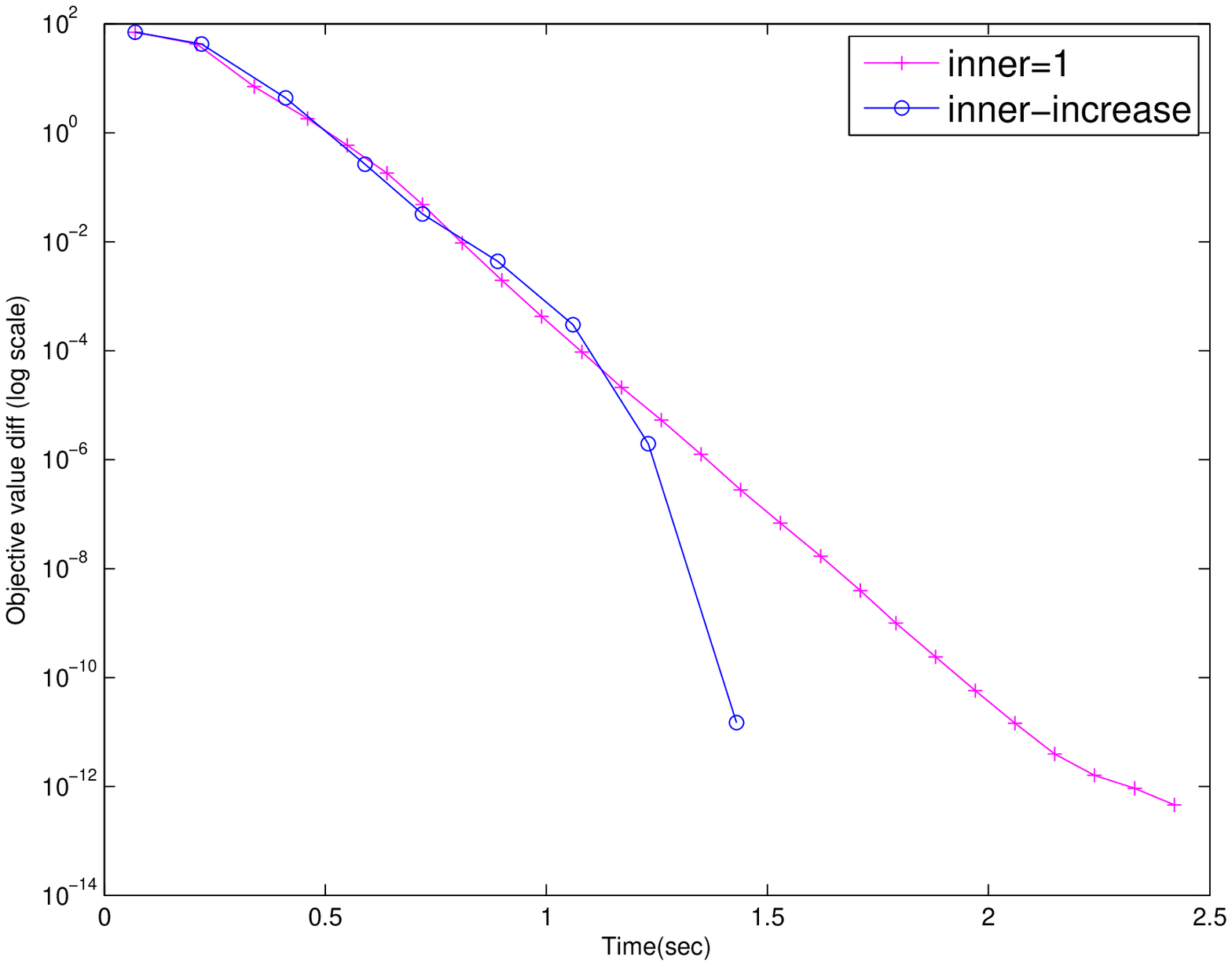}}
\end{tabular}
\caption{The behavior of \QUIC when varying the number of inner iterations. Figure \ref{fig:stop_1} show that \QUIC with one inner iteration converges faster in the beginning but eventually achieves just linear convergence, while \QUIC with 20 inner iterations converges slower in the beginning, but has quadratic convergence.  Figure \ref{fig:stop_2} shows that by adaptively setting the number of iterations in \QUIC, we get the advantages of both cases.}
\label{fig:stopping}
\end{figure}

We can observe that \QUIC with one inner iteration converges faster in the beginning, but eventually achieves just a linear convergence rate, while \QUIC with 20 inner iterations converges more slowly in the beginning, but eventually achieves quadratic convergence. Based on this observation, we propose an adaptive stopping condition: we set the number of coordinate descent steps to be $\alpha t$ for the $t$-th outer iteration, where $\alpha$ is a constant; we use $\alpha = 1/3$ in our experiments. Figure \ref{fig:stop_2} shows that by using this adaptive stopping condition, \QUIC is not only efficient in the beginning, but also achieves quadratic convergence.

\subsection{Comparisons with other methods}

In this section, we compare the performance of \QUIC on both synthetic and real datasets with other state-of-the-art methods. We have implemented \QUIC in C++, and all experiments were executed on 2.83GHz Xeon X5440 machines with 32G RAM and Linux OS.

We include the following algorithms in our comparisons: 
\begin{itemize}
  \item \ALM: the Alternating Linearization Method proposed by \cite{ALM}. We use their MATLAB source code for the experiments. 
  \item \ADMM: another implementation of the alternating linearization method 
    implemented by \cite{ADMM}. The matlab code can be downloaded from \\ 
    \url{http://www.stanford.edu/~boyd/papers/admm/}. We found that the 
    default parameters (which we note are independent of the 
    regularization penalty) yielded slow convergence; we set the augmented 
    Lagrangian parameter to $\rho=50$ and the over-relaxation parameter to 
    $\alpha=1.5$.  These parameters achieved the best speed on the ER dataset.
  \item \glasso: the block coordinate descent method proposed by \cite{Friedman:glasso:2007}. We use the latest version \glasso 1.7 downloaded from \\ \url{http://www-stat.stanford.edu/~tibs/glasso/}.
  \item \PSM: the Projected Subgradient Method proposed by \cite{JD08a}. We use the MATLAB source code available at \\
    \url{http://www.cs.ubc.ca/~schmidtm/Software/PQN.html}. 
  \item \SINCO: the greedy coordinate descent method proposed by \cite{SINCO}. The code can be downloaded from \\ 
    \url{https://projects.coin-or.org/OptiML/browser/trunk/sinco}.
  \item \IPM: An inexact interior point method proposed by \cite{LL10a}. The source code can be downloaded from \\ 
    \url{http://www.math.nus.edu.sg/~mattohkc/Covsel-0.zip}.
  \item \PQN: the projected quasi-Newton method proposed by \cite{Schmidt:PQN:2009}. The source code can be downloaded from \\ \url{http://www.di.ens.fr/~mschmidt/Software/PQN.html}. 
\end{itemize}

In the following, we compare \QUIC and the above state-of-the-art methods on synthetic and real datasets with various settings of $\lambda$.

\subsubsection{Experiments on synthetic datasets}
\label{sec:synthetic}

We first compare the run times of the different methods on synthetic data. We generate the two following types of graph structures for the underlying Gaussian Markov Random Fields:
\begin{itemize}
\item Chain Graphs: The ground truth inverse covariance matrix $\Sigma^{-1}$ is set to be $\Sigma^{-1}_{i,i-1}=-0.5$ and $\Sigma^{-1}_{i,i}=1.25$.
\item Graphs with Random Sparsity Structures: We use the procedure mentioned in Example 1 in \cite{LL10a} to generate inverse covariance matrices with random non-zero patterns.  Specifically, we first generate a sparse matrix $U$ with nonzero elements equal to $\pm 1$, set $\Sigma^{-1}$ to be $U^T U$ and then add a diagonal term to ensure $\Sigma^{-1}$ is positive definite. We control the number of nonzeros in $U$ so that the resulting $\Sigma^{-1}$ has approximately $10p$ nonzero elements.
\end{itemize}
Given the inverse covariance matrix $\Sigma^{-1}$, we draw a limited number, $n = p/2$ i.i.d. samples from the corresponding GMRF distribution, in order to simulate the high-dimensional setting.

Table~\ref{tab:synth-datasets} shows the attributes of the synthetic datasets 
that we used in the timing comparisons. The dimensionality varies from 
$\{1000,4000,10000\}$. For chain graphs, we select $\lambda$ so that the 
solution has (approximately) the correct number of nonzero elements. In order 
to test the performance of the algorithms under different values of $\lambda$, 
for the case of random-structured graphs we considered two $\lambda$ values; 
one of which resulted in the discovery of the correct number of non-zeros and one 
which resulted in five-times thereof.  We measured the accuracy of the graph 
structure recovered by the 
true positive rate (TPR) and false positive rate (FPR) 
defined as
\begin{align} 
  \text{TPR} = \frac{|\{(i,j) \mid (X^*)_{ij}>0 \text{ and } 
  Q_{ij}>0\}|}{|\{(i,j)\mid Q_{ij}>0\}|},
  \text{FPR} &= \frac{|\{(i,j) \mid (X^*)_{ij}>0 \text{ and } Q_{ij}=0\}|}{|\{(i,j)\mid Q_{ij}=0\}|},
\label{eqn:TP-FP}
\end{align}
where $Q$ is the ground truth sparse inverse covariance.
\begin{table*}[tb]
  \centering
  \caption{The parameters and solution properties of the synthetic datasets. 
  $p$ stands for dimension, $\|\Sigma^{-1}\|_0$ indicates the number of 
  nonzeros in ground truth inverse covariance matrix, $\|X^*\|_0$ is the 
  number of nonzeros in the solution. TPR/FPR is the true and false recovery 
  rate, see~\eqref{eqn:TP-FP} \label{tab:synth-datasets}. }
  \begin{tabular}{|c|c|c|c|c|c|c|}
    \hline
    \multicolumn{3}{|c|}{Dataset} & {Parameter} & 
    \multicolumn{3}{|c|}{Solution properties} \\
    \hline
    pattern & $p$ &  $\|\Sigma^{-1}\|_0$ & $\lambda$ & $\|X^*\|_0$ 
    & TP & FP \\
    \hline
    {chain} & {$1000$} &{$2998$} 
    &{$0.4$} &  $3028$ & $1$ & $3\times 10^{-5}$ \\
    \hline
    {chain} &{$4000$} & {$11998$} 
    & {$0.4$} & $11998$ & $1$ & $0$ \\
    \hline
    {chain} &{$10000$} &{$29998$} 
    &{$0.4$} & $29998$ & $1$ & $0$ \\
    \hline
    {random} & \multirow{2}{*}{$1000$} & \multirow{2}{*}{$10758$} 
    & {$0.12$}  & $10414$ & $0.69$ & $4\times 10^{-3}$ \\
    \cline{4-7}
    & & &{$0.075$} & $55830$ & $0.86$ & $0.05$ \\
    \hline
    \multirow{2}{*}{random} & \multirow{2}{*}{$4000$} &  $\multirow{2}{*}{$41112$}$ 
    & {$0.08$} & $41936$ & $0.83$ & $6\times 10^{-3}$ \\
    \cline{4-7}
    & & & {$0.05$} &  $234888$ & $0.97$ & $0.05$ \\
    \hline
    \multirow{2}{*}{random} & \multirow{2}{*}{$10000$} &  $\multirow{2}{*}{$91410$}$ 
    & {$0.08$} & $89652$ & $0.90$ & $4\times 10^{-6}$ \\
    \cline{4-7}
    & & & {$0.04$} & $392786$ & $1$ & $3\times 10^{-3}$ \\
    \hline
  \end{tabular}
\end{table*}

Since \QUIC does not natively compute a dual solution, the duality gap cannot be used as a stopping condition.\footnote{Note, that $W = X^{-1}$ cannot be expected to satisfy the dual constraints requiring $|W_{ij} - S_{ij}| \leq \lambda_{ij}$. One could project $X^{-1}$ in order to enforce the constraints and use the resulting matrix to compute the duality gap. Our implementation provides this computation only if the user requests it.} In practice, we can use the minimum-norm sub-gradient defined in Definition \ref{def:subgradient} as the stopping condition. There is no additional computational cost to this approach because $X^{-1}$ is computed as part of the \QUIC algorithm. In the experiments, we report the time for each algorithm to achieve $\epsilon$-accurate solution defined by $f(X^k)-f(X^*)<\epsilon f(X^*)$.
\begin{table*}[tb]
  \centering
  \caption{Running time comparisons on synthetic datasets. See also Table~\ref{tab:synth-datasets} regarding the dataset properties. We used $*$ to indicate that the run time exceeds 30,000 seconds (8.3 hours).  The results show that \QUIC is overwhelmingly faster than other methods, and is the only one which is able to scale up to solve problem where $p=10000$.   \label{tab:synth-compare}}
{\renewcommand{\tabcolsep}{0.12cm}
  \begin{tabular}{|c|c|c|c|r|r|r|r|r|r|r|r|}
    \hline
    \multicolumn{4}{|c}{Parameters} & \multicolumn{8}{|c|}{Time (in seconds)} \\
    \hline
    pattern & $p$ & $\lambda$ & $\epsilon$ &\QUIC & \ALM & Glasso & \PSM & \IPM & Sinco & PQN & \small ADMM \\
    \hline
    \multirow{2}{*}{chain} & \multirow{2}{*}{$1000$} & \multirow{2}{*}{$0.4$} 
    & $10^{-2}$ & { $\bf <1$} & $19$ & $9$ & $16$ & $86$ & $120$ & $110$ 
    & $62$ \\
     & & & $10^{-6}$  & {\bf 2} & $42$ & $20$ & $35$ & $151$ & $521$ & $210$& $281$\\
    \hline
    \multirow{2}{*}{chain} & \multirow{2}{*}{$4000$} & \multirow{2}{*}{$0.4$} & $10^{-2}$ & {\bf 11} & $922$ & $460$ & $568$ & $3458$ & $5246$ & $672$ & $1028$ \\
    & & & $10^{-6}$ & {\bf 54} & $1734$ & $1371$ & $1258$ & $5754$ & * & $10525$ & $2584$ \\
    \hline
    \multirow{2}{*}{chain} & \multirow{2}{*}{$10000$} & \multirow{2}{*}{$0.4$} & $10^{-2}$ & {\bf 217} & $13820$ & $10250$ & $8450$ & * & * & * & * \\
    & & & $10^{-6}$ & {\bf 987} & $28190$ & * & $19251$ & * & * & * & * \\
    \hline
    \multirow{4}{*}{random} & \multirow{4}{*}{$1000$} & \multirow{2}{*}{$0.12$} 
    & $10^{-2}$ & {$ \bf <1$} & $42$ & $7$ & $20$ & $72$ & $61$ & $33$& $35$ \\
    & & & $10^{-6}$ & {\bf 1} & $28250$ & $15$ & $60$ & $117$ & $683$ & $158$ & $252$ \\
    \cline{3-12}
    & & \multirow{2}{*}{$0.075$} & $10^{-2}$ & {\bf 1} & $66$ & $14$ & $24$ & $78$ & $576$ & $15$ & $56$ \\
    & & & $10^{-6}$ & {\bf 7} & *  & $43$ & $92$ & $146$ & $4449$& $83$ & *\\
    \hline
    \multirow{4}{*}{random} & \multirow{4}{*}{$4000$} & \multirow{2}{*}{$0.08$} & $10^{-2}$ & {\bf 23} & $1429$ & $864$ & $1479$ & $4928$ & $7375$ & $2052$& $1025$ \\
    & & & $10^{-6}$ & {\bf 160} & * & $1743$ & $4232$ & $8097$ & * & $4387$ & * \\
    \cline{3-12}
    & & \multirow{2}{*}{$0.05$} & $10^{-2}$ & {\bf 66} & * & $2514$ & $2963$ & $5621$ & * & $2746$ & * \\
    & & & $10^{-6}$ & {\bf 479} & *  & $5712$ & $9541$ & $13650$ & * & $8718$ & *\\ 
    \hline
    \multirow{4}{*}{random} & \multirow{4}{*}{$10000$} & \multirow{2}{*}{$0.08$} & $10^{-2}$ & {\bf 338} & $26270$ & $14296$ & * & * & * & * & *\\
    & & & $10^{-6}$ & {\bf 1125} & * & * & * & * & * & *& *\\
    \cline{3-12}
    & & \multirow{2}{*}{$0.04$} & $10^{-2}$ & {\bf 804} & * & * & * & * &* & * & *\\
    & & & $10^{-6}$ & {\bf 2951} & *  & * & * & * &  * &  * &* \\
    \hline
  \end{tabular}
}
\end{table*}

\begin{figure}[t!]
\begin{tabular}{cc}
\subfigure[Objective value versus time on chain1000
]{\includegraphics[width=0.49\textwidth]{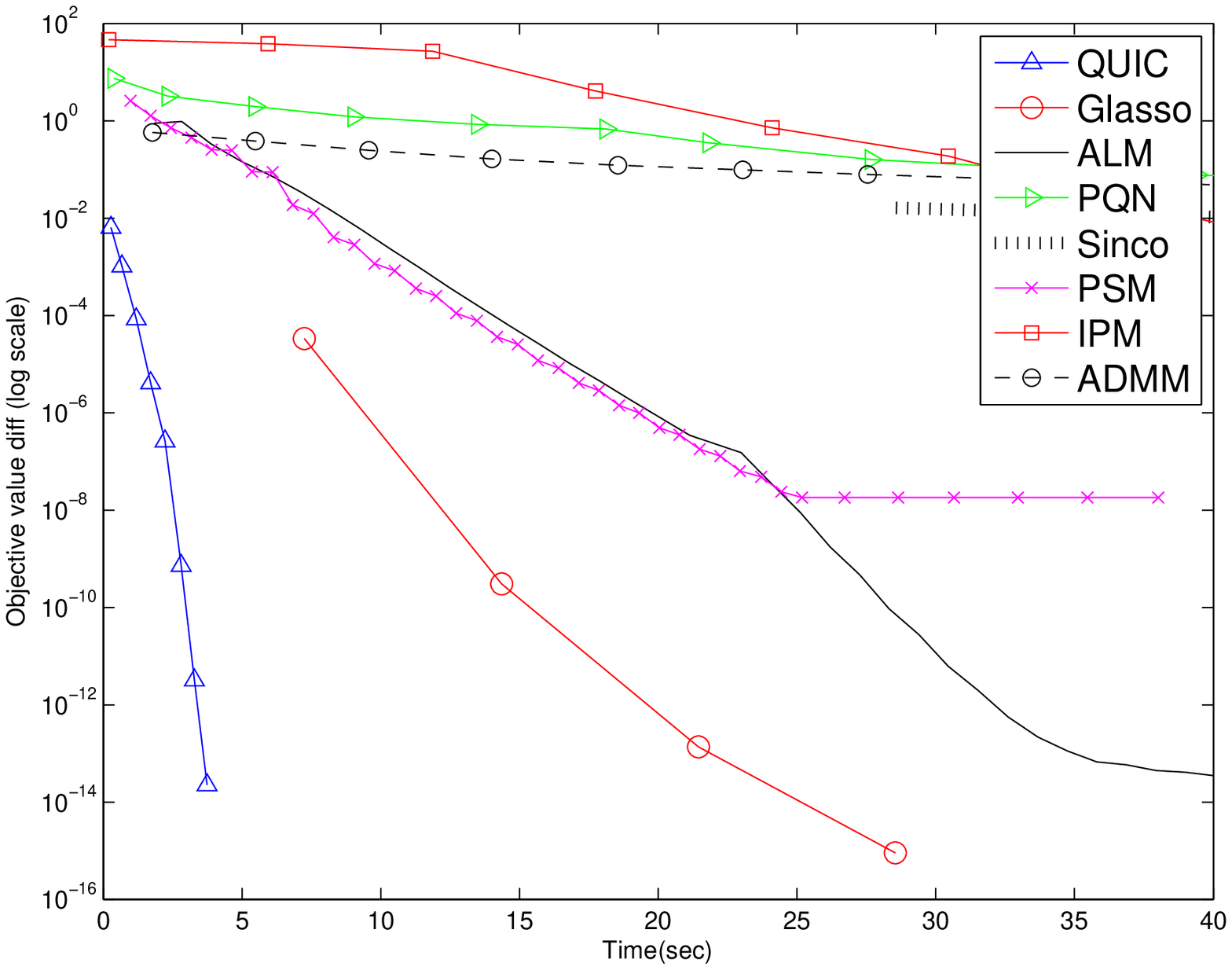}}
&
\subfigure[Objective value versus time on random1000
]{\includegraphics[width=0.49\textwidth]{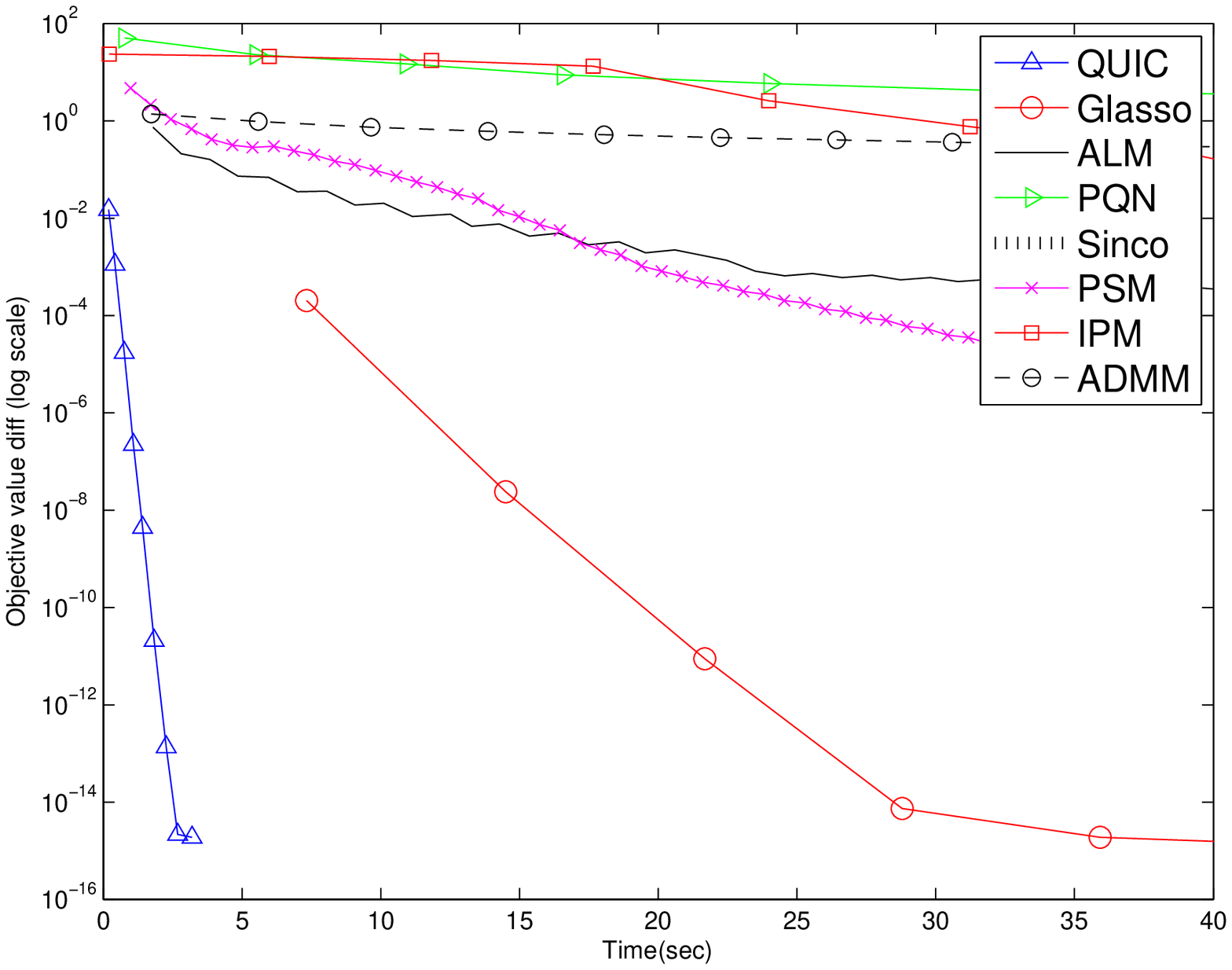}}
\\ 
\subfigure[True positive rate versus time on chain1000 
]{\includegraphics[width=0.49\textwidth]{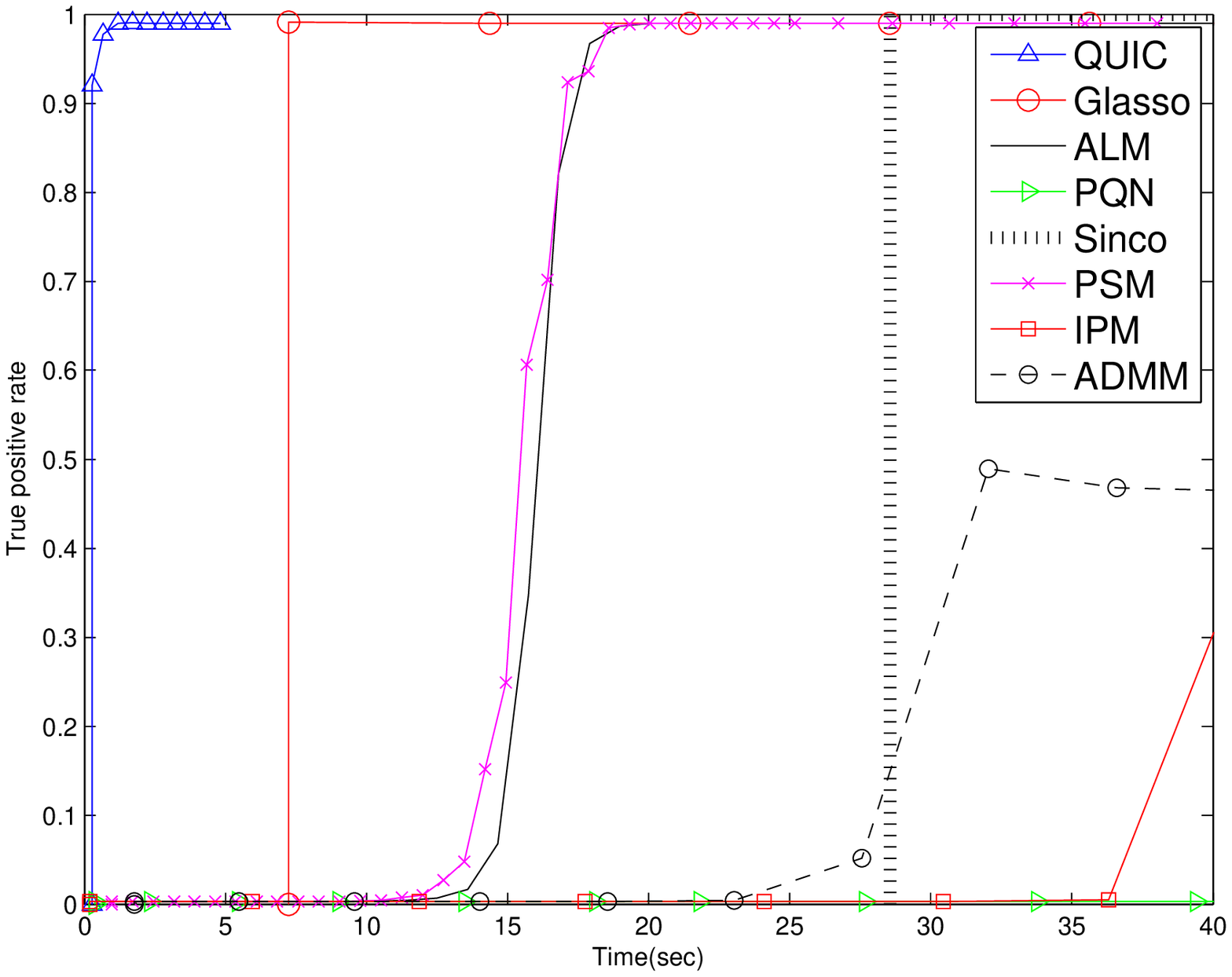}}
& 
\subfigure[True positive rate versus time on random1000 
]{\includegraphics[width=0.49\textwidth]{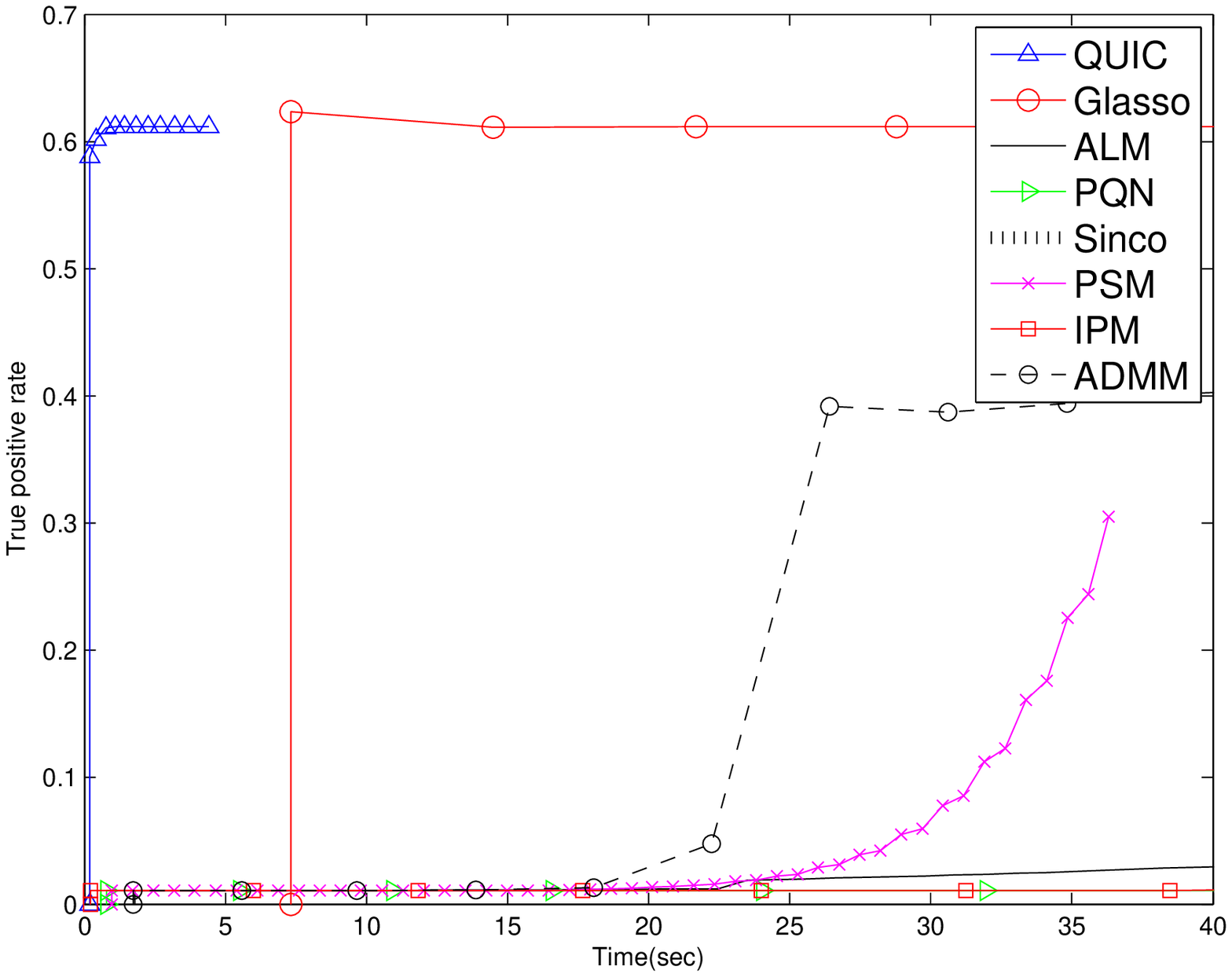}}
\\
\subfigure[False positive rate versus time on chain1000
]{\includegraphics[width=0.49\textwidth]{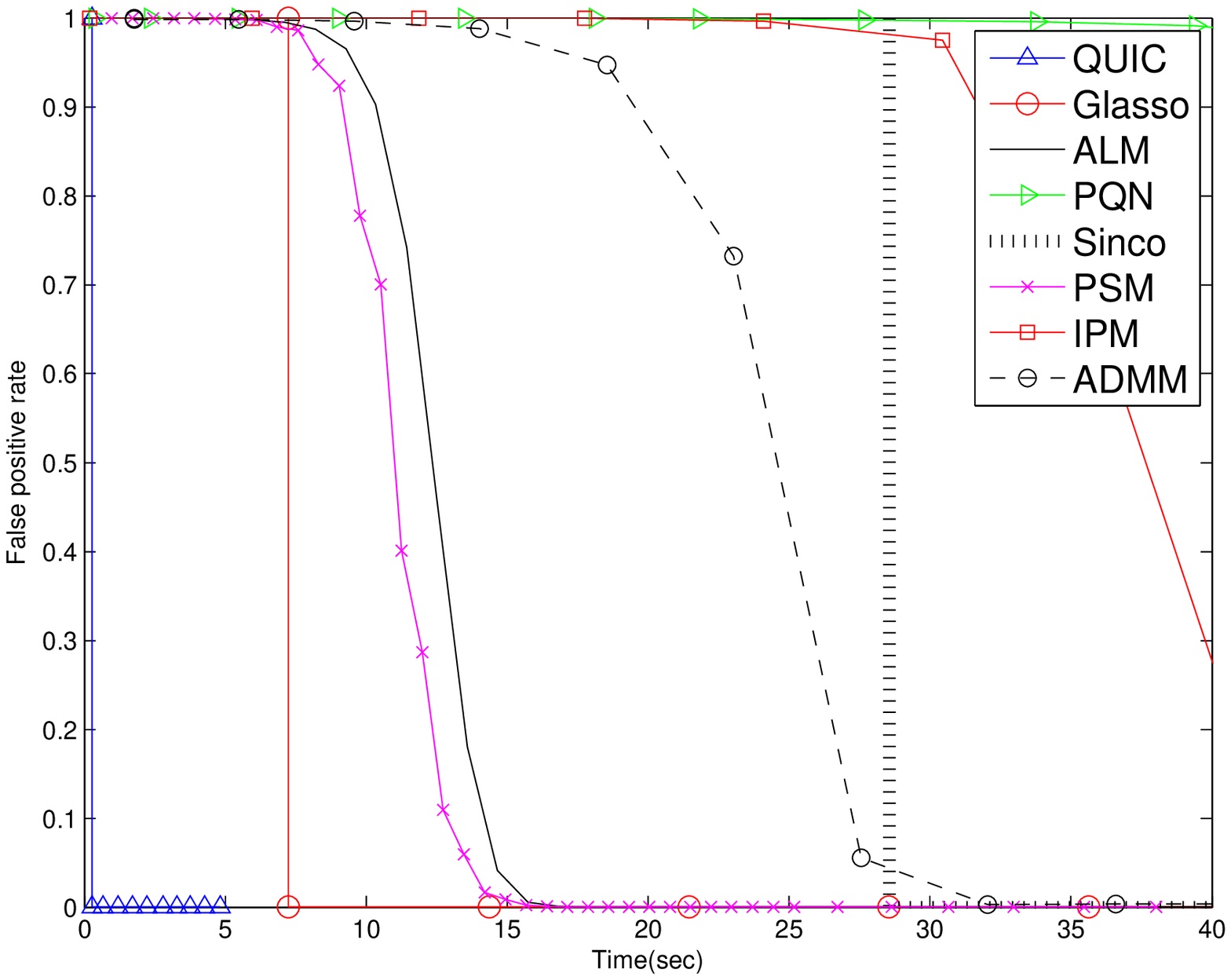}}
& 
\subfigure[False positive rate versus time on random1000
]{\includegraphics[width=0.49\textwidth]{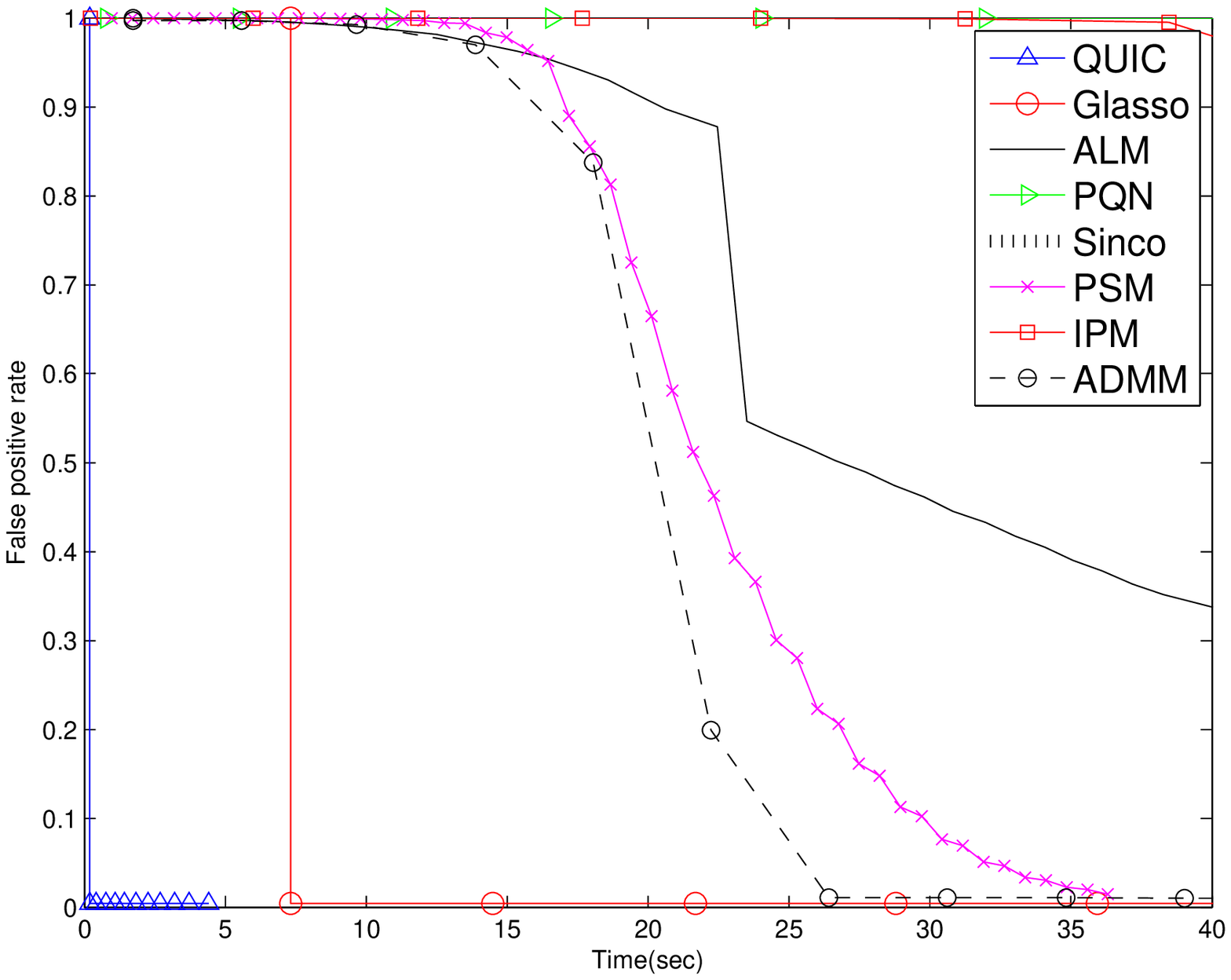}}
\end{tabular}
\caption{Comparison of algorithms on two synthetic dataset: chain1000 and random1000. We can see that \QUIC achieves a solution with better objective function value as well as true positive rate and false positive rate in both datasets.}
\label{fig:tpfp}
\end{figure}

Table \ref{tab:synth-compare} shows the results for $\epsilon=10^{-2}$ and $10^{-6}$, where $\epsilon=10^{-2}$ tests the ability for an algorithm to get a good initial guess (the nonzero structure), and $\epsilon=10^{-6}$ tests whether an algorithm can achieve an accurate solution. Table~\ref{tab:synth-compare} shows that \QUIC is consistently and overwhelmingly faster than other methods, both initially with $\epsilon=10^{-2}$, and at $\epsilon = 10^{-6}$.  Moreover, for the $p=10000$ random pattern, there are $p^2=$ 100 million variables and the selection of fixed/free sets helps \QUIC to focus only on very small subset of them. We converge to the solution in about 15 minutes, while other methods fail to even have an initial guess within 8 hours.

In some applications, researchers are primarily interested in the obtained graphical structure represented by the solution. Therefore, in addition to the objective function value, we further compare the true positive rate and false positive rate of the obtained nonzero pattern in $X_t$ by each algorithm.  In Figure \ref{fig:tpfp}, we use two synthetic datasets, chain1000 and random1000, as examples. For each algorithm, we plot the objective function value, true positive rate, and false positive rate of the iterates $X_t$ versus run time. For the methods that solve the dual problem, the sparse inverse covariance matrix $X_t=W_t^{-1}$ is usually dense, so we consider elements with absolute value larger than $10^{-6}$ as nonzero elements.  We can see that \QUIC not only obtains lower objective function value efficiently, but also recover the ground truth structure of GMRF faster than other methods.

\subsubsection{Experiments on real datasets}
We use the real world biology datasets preprocessed by \cite{LL10a} to compare 
the performance of our method with other state-of-the-art methods. In the 
first set of experiments, we set the regularization parameter $\lambda$ to be 
$0.5$, which achieves reasonable sparsity for the following datasets: Estrogen 
($p=692$), Arabidopsis ($p=834$), Leukemia ($p=1,225$), Hereditary ($p=1,869$).  
In Figure \ref{fig:real} we plot the relative error $(f(X_t)-f(X^*))/f(X^*)$ 
(on a log scale) against time in seconds.  We can observe from Figure 
\ref{fig:real} that under this setting -- large $\lambda$ and sparse solution 
-- \QUIC can be seen to achieve super-linear convergence while other methods 
exhibit at most a linear convergence.  Overall, we see that \QUIC can be five 
times faster than other methods, and can be expected to be even faster if a  
higher accuracy is desired.

\begin{figure}[htp!]
\begin{tabular}{cc}
  \subfigure[Time taken on ER dataset, $p=692$, $\frac{\|\Theta^*\|_0}{p^2}=0.0222$
  \label{fig:real_ER}]{\includegraphics[width=0.49\textwidth]{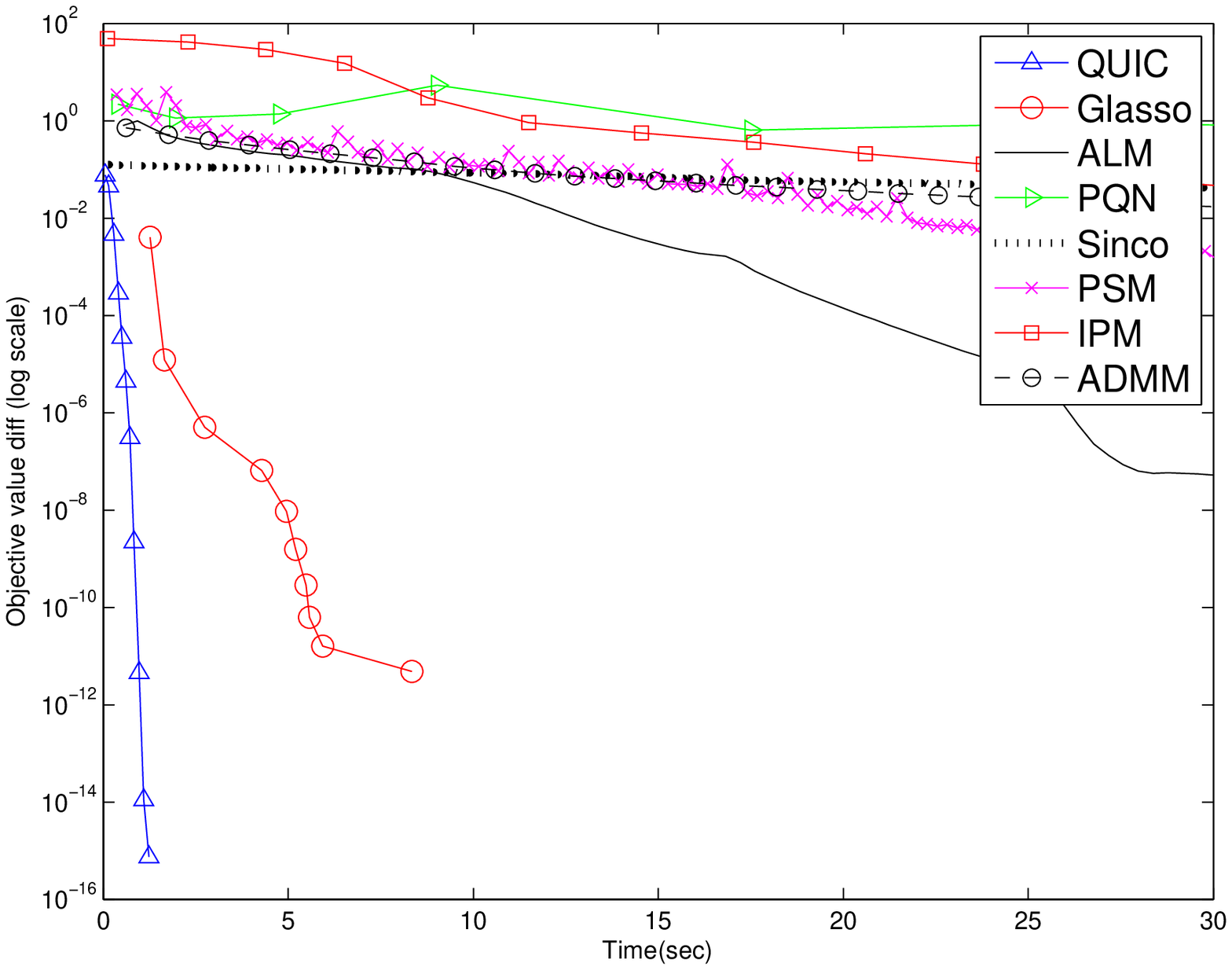}}
&
\subfigure[Time taken on Arabidopsis dataset, $p=834$, $\frac{\|\Theta^*\|_0}{p^2}=0.0296$]{\includegraphics[width=0.49\textwidth]{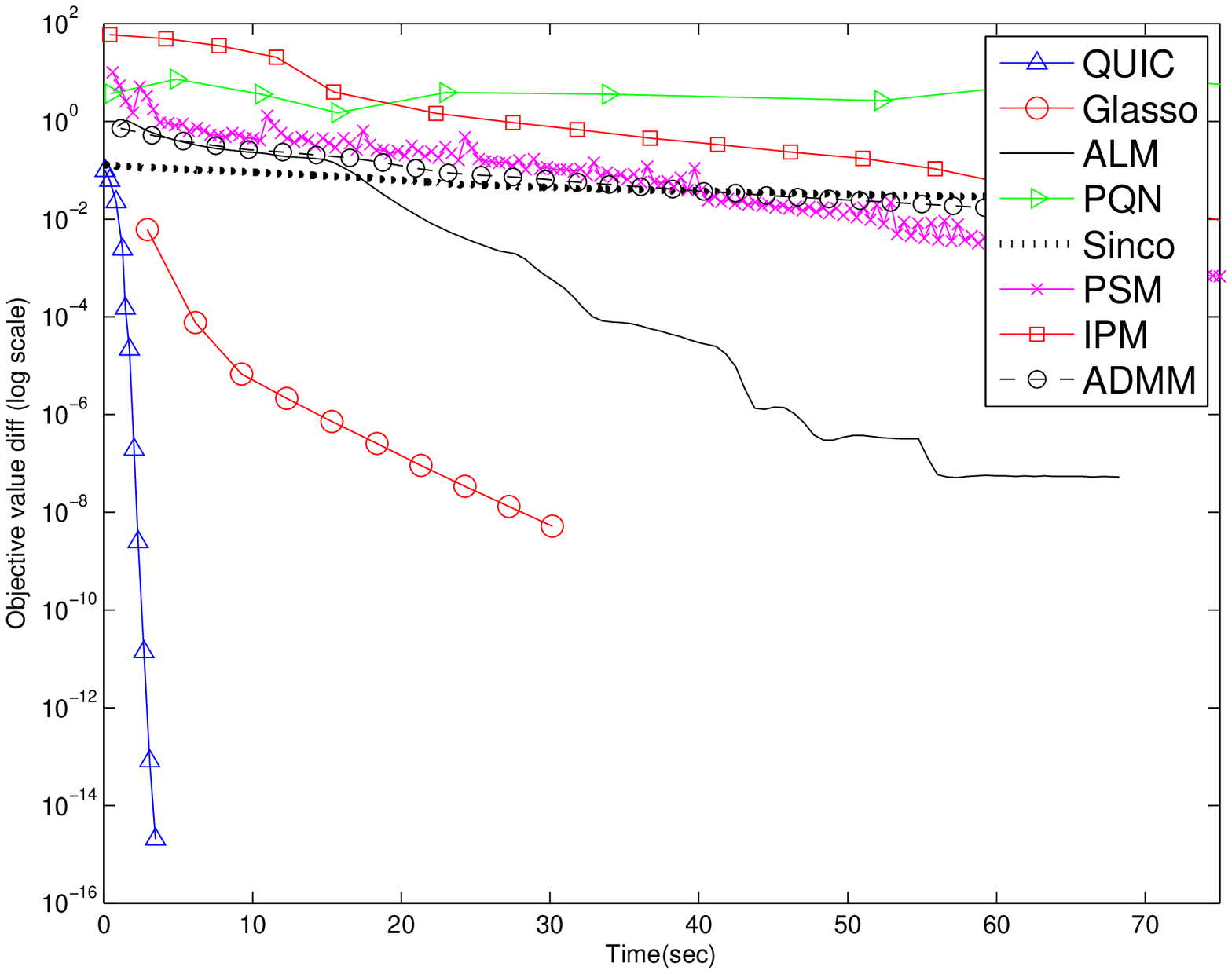}}
\\
\subfigure[Time taken on Leukemia dataset, $p=1,255$, 
$\frac{\|\Theta^*\|_0}{p^2}=0.0221$
]{\includegraphics[width=0.49\textwidth]{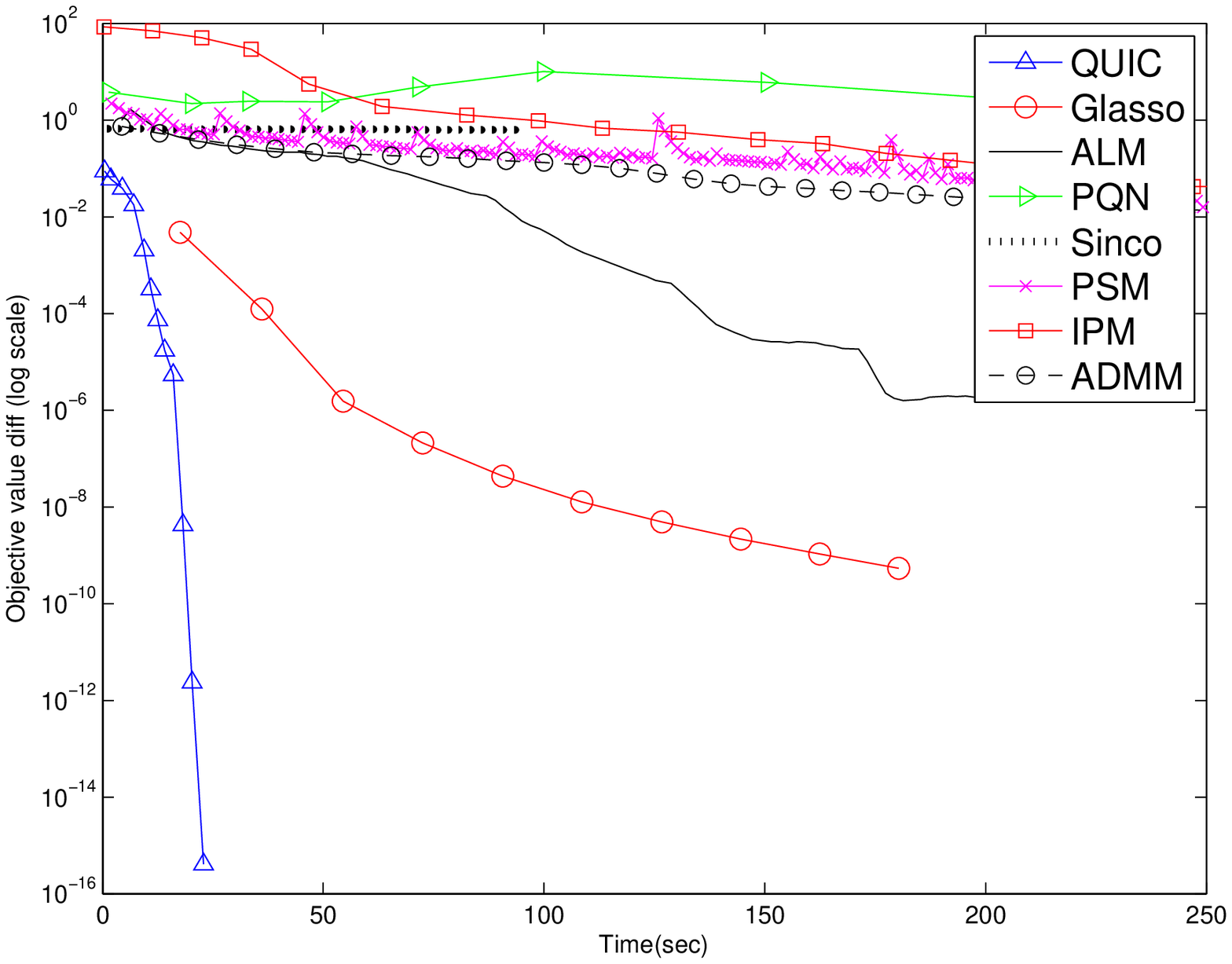}}
&
\subfigure[Time taken on hereditarybc dataset, $p=1,869$, 
$\frac{\|\Theta^*\|_0}{p^2}=0.0198$]{\includegraphics[width=0.49\textwidth]{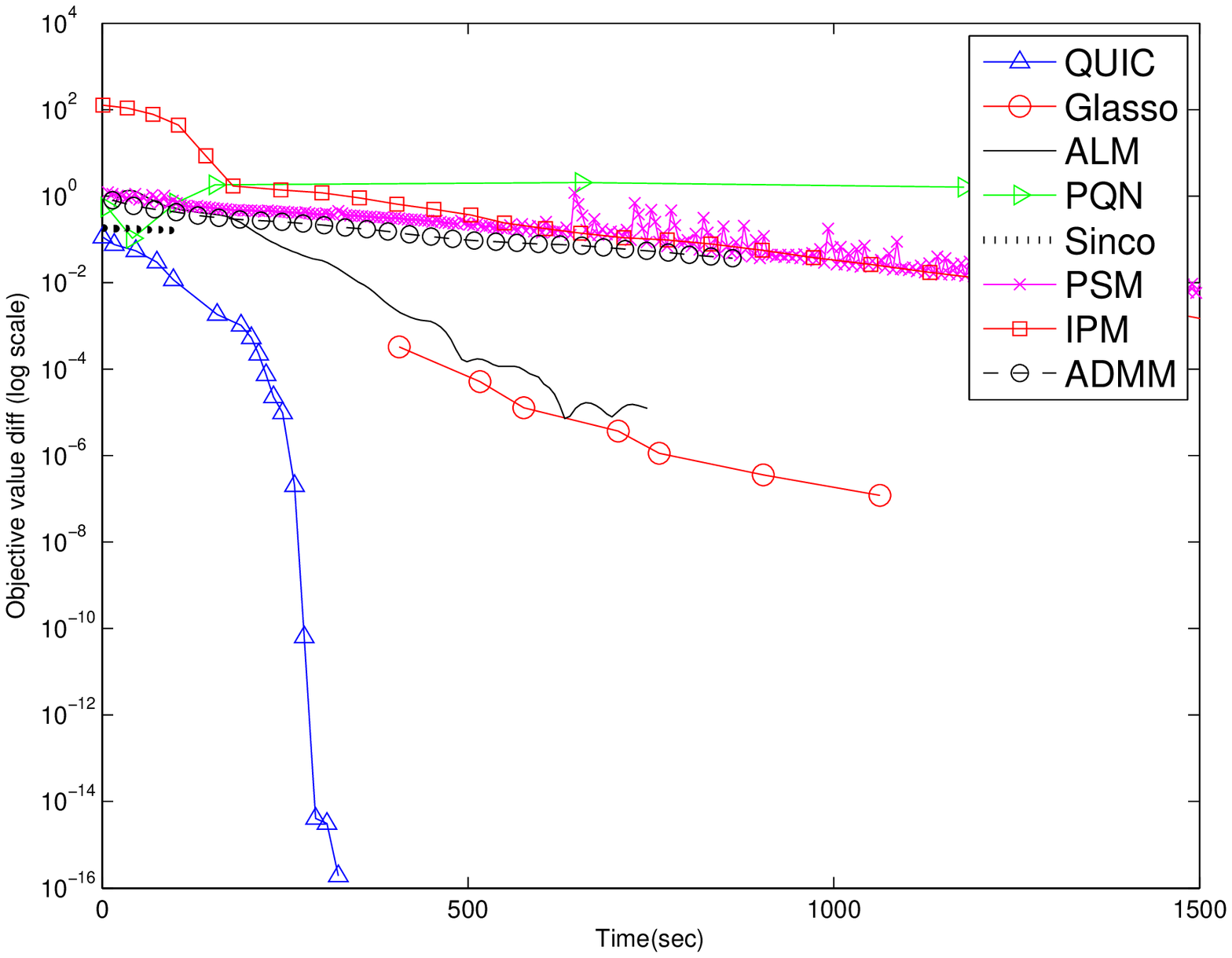}} 
\end{tabular}
\caption{Comparison of algorithms on real datasets with $\lambda=0.5$. The results show QUIC
converges faster than other methods.  
}
\label{fig:real}
\end{figure}

In the second set of experiments, we compare the algorithms under different 
values of the regularization parameter $\lambda$ on the ER dataset. In Figure 
\ref{fig:real_ER} we show the results for $\lambda=0.5$. We then 
decrease $\lambda$ to $0.1$, $0.05$, $0.01$ using the same datasets and 
show the results in Figure \ref{fig:real_lambda}.  A smaller $\lambda$ yields 
a  
denser solution, and we list the density of the convergence point $X^*$ in 
Figure \ref{fig:real_lambda}. From Figure \ref{fig:real_lambda} we can see 
that \QUIC is the most efficient method when $\lambda$ is large (solution is 
sparse), but \IPM and \PSM outperforms \QUIC when $\lambda$ is small (solution 
is dense). However, those cases are usually not useful in practice because 
when solving the $\ell_1$-regularized MLE problem one usually wants a sparse 
graphical structure for the GMRF. The main reason that \QUIC is very efficient 
under large $\lambda$ is that with {\em fixed}/{\em free} set selection, the 
coordinate descent method can focus on a small portion of variables, while in 
\PSM and \IPM the whole matrix is considered at each iteration.

To further demonstrate the power of {\em fixed}/{\em free} set selection, we use Hereditarybc dataset as an example.  In Figure \ref{fig:freeset}, we plot the size of the free set versus iterations for Hereditarybc dataset. Starting from a total of $1869^2=3,493,161$ variables, the size of the free set progressively drops, in fact to less than $120,000$ in the very first iteration. We can see the super-linear convergence of \QUIC even more clearly when we plot it against the number of iterations.

\begin{figure}[ht]
\centering
\includegraphics[width=0.7\textwidth]{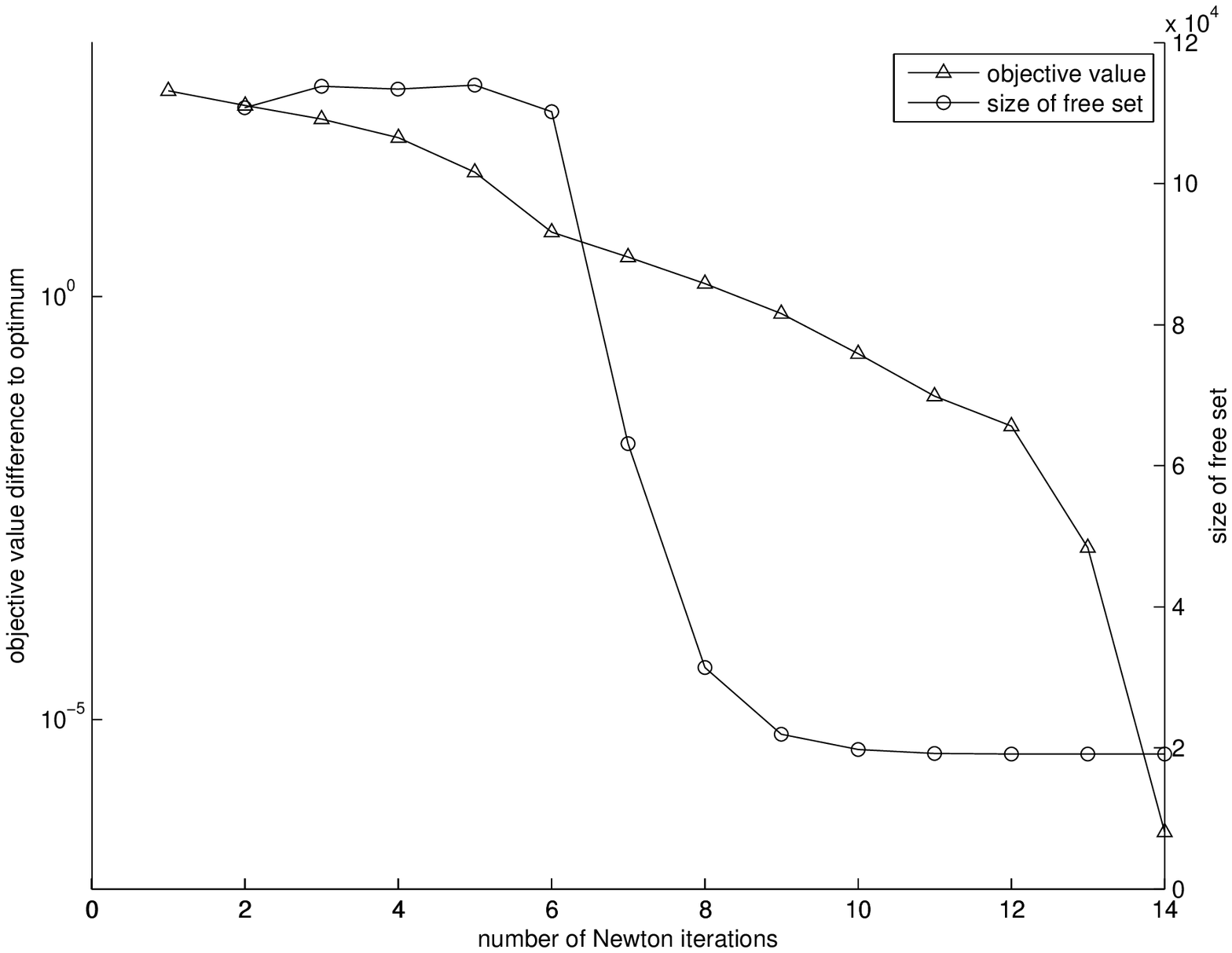}
  \caption{Size of free sets and objective value versus iterations (Hereditarybc dataset). There are total of $3,493,161$ variables, but the size of the free set reduces to less than $120,000$ in one iteration, becoming about $20,000$ at the end. }
  \label{fig:freeset}
\end{figure}
\begin{figure}[htp!]
\begin{tabular}{cc}
\subfigure[Time taken on ER dataset, $\lambda=0.1$, $\frac{\|\Theta^*\|_0}{p^2}=0.0724$
]{\includegraphics[width=0.49\textwidth]{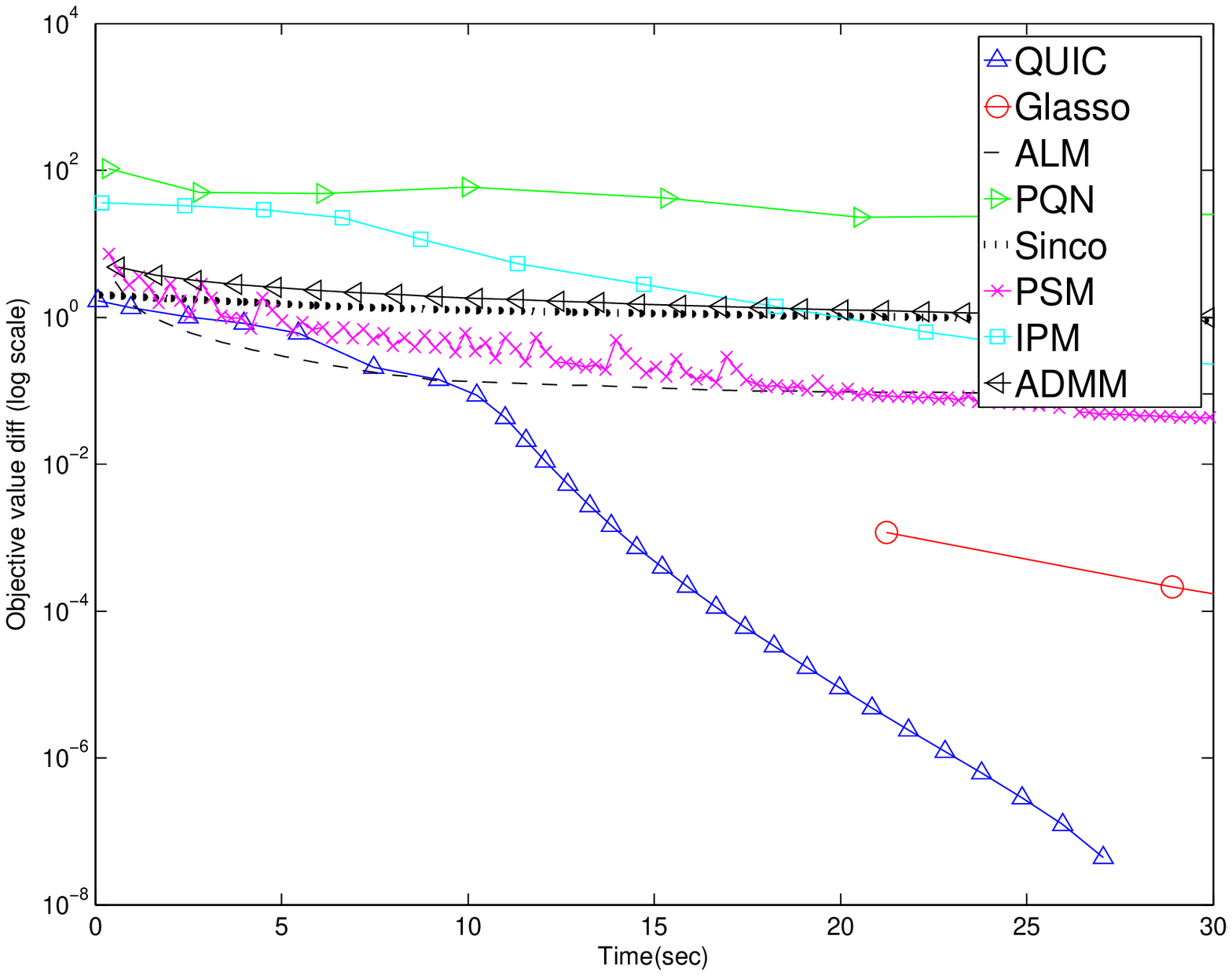}}
&
\subfigure[Time taken on ER dataset, $\lambda=0.05$, $\frac{\|\Theta^*\|_0}{p^2}=0.125$  
]{\includegraphics[width=0.49\textwidth]{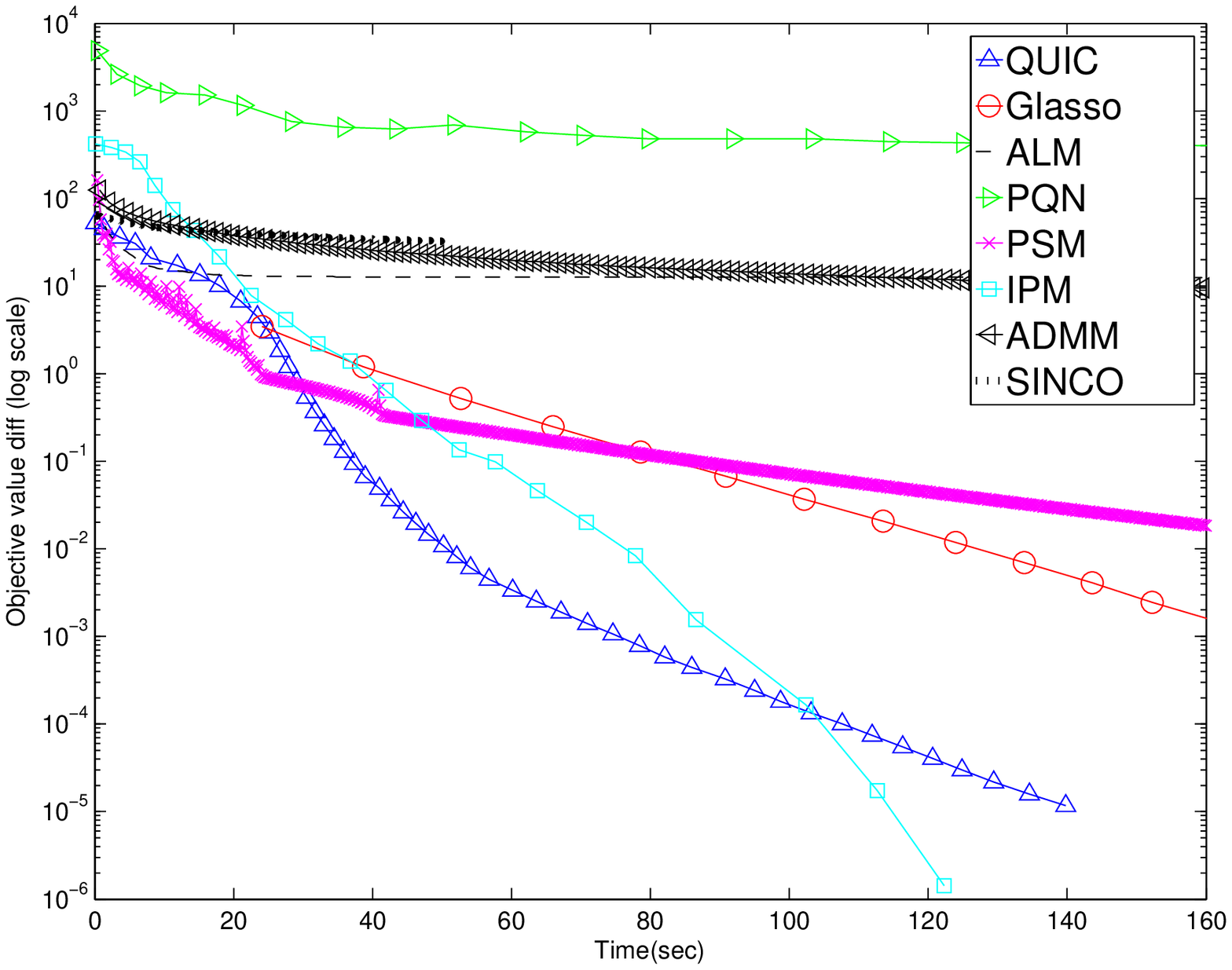}}
\\
  \subfigure[Time taken on ER dataset, $\lambda=0.01$, $\frac{\|\Theta^*\|_0}{p^2}=0.322$
]{\includegraphics[width=0.49\textwidth]{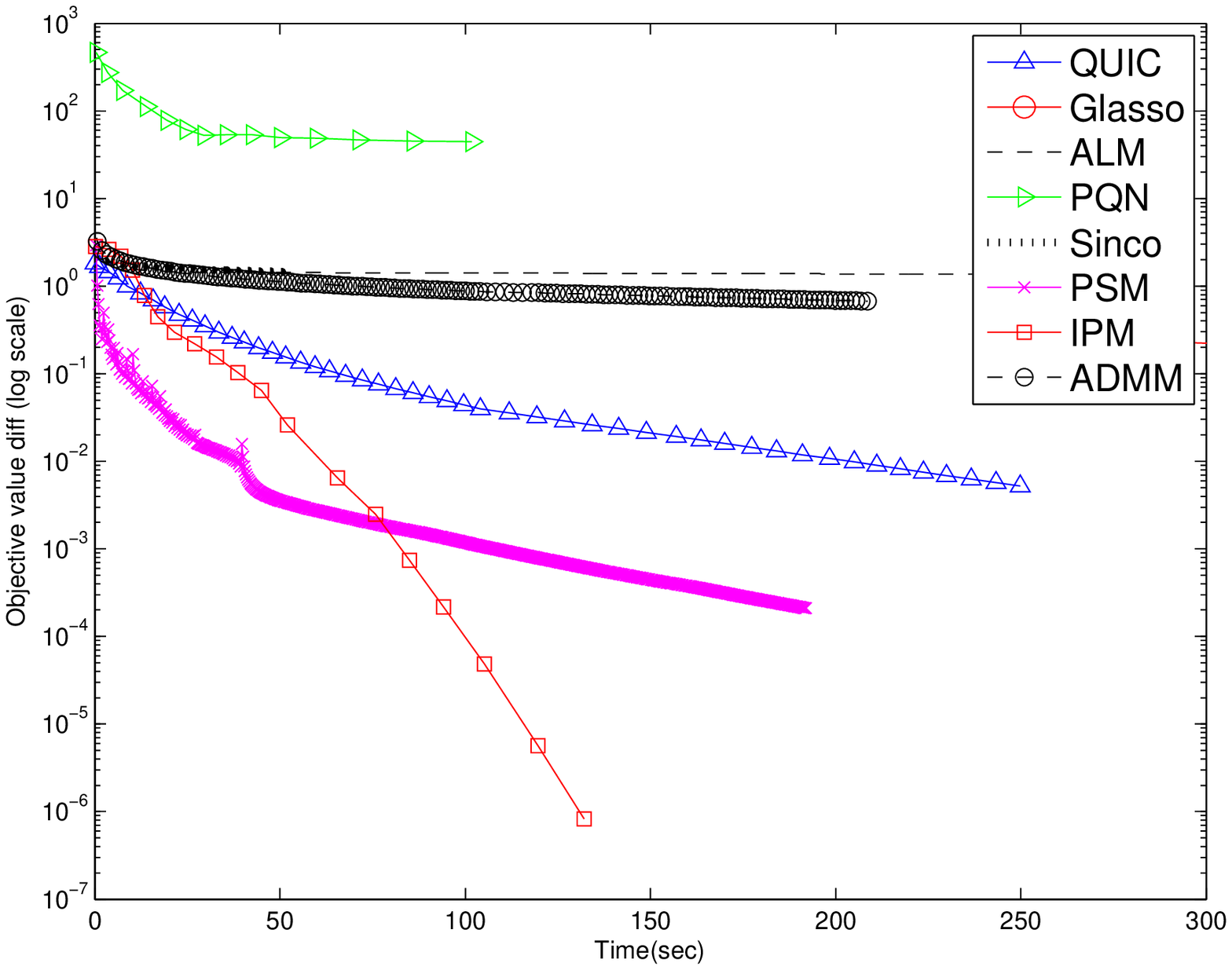}}
&
\subfigure[Time taken on ER dataset, $\lambda=0.005$, $\frac{\|\Theta^*\|_0}{p^2}=0.487$
]{\includegraphics[width=0.49\textwidth]{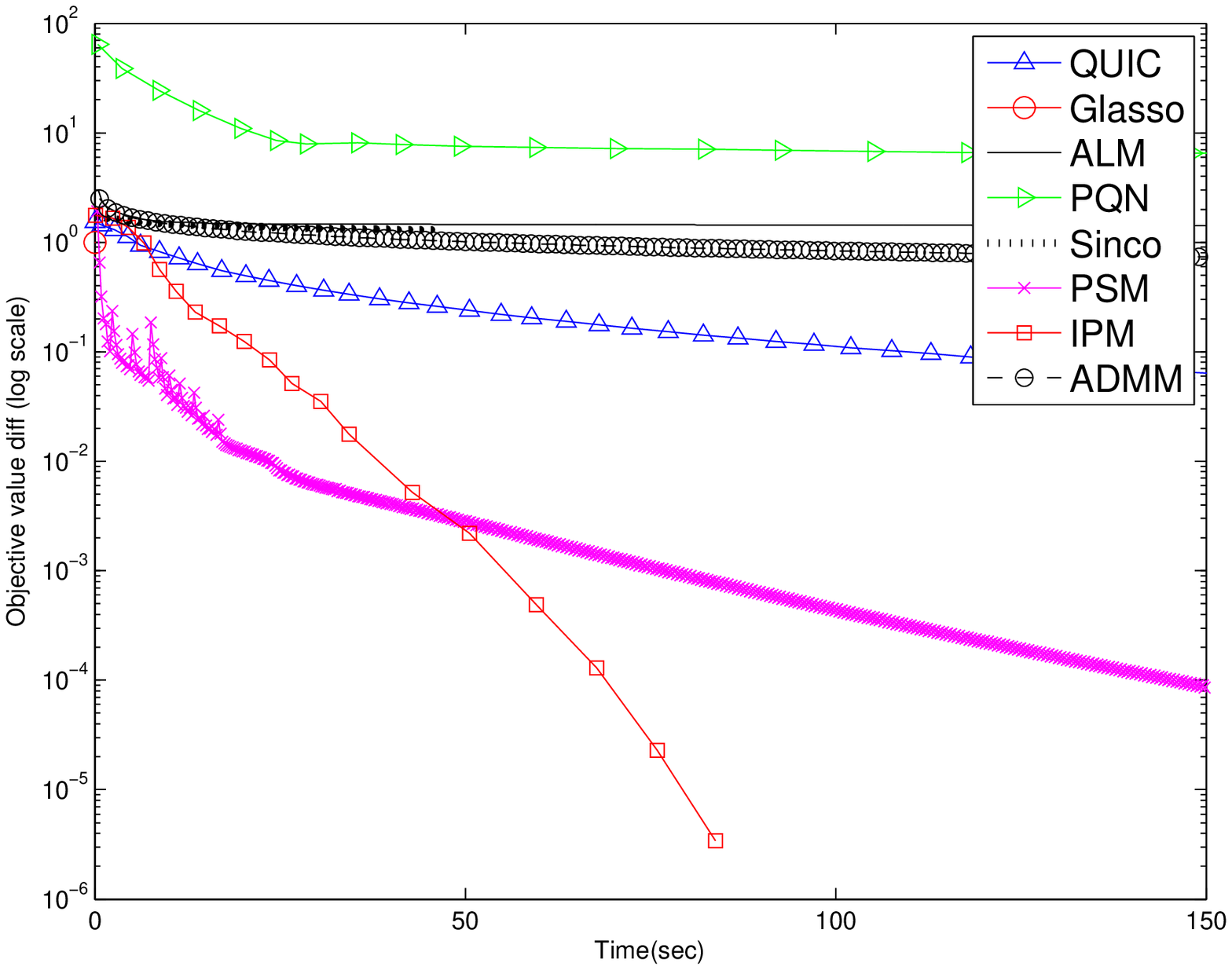}} 
\end{tabular}
\caption{Comparison of algorithms on ER dataset ($p=692$) under different $\lambda$. The results show \QUIC converges faster for larger $\lambda$ where solutions are sparse, while \IPM and \PSM are faster for smaller $\lambda$ which produces denser solutions.}
\label{fig:real_lambda}
\end{figure}

\subsection{Block-diagonal structure}

As discussed earlier, \cite{hastie:2012,friedman:2011} showed that when the 
thresholded covariance matrix $E=\max(|S|-\lambda, 0)$ is block-diagonal, then 
the problem can be naturally decomposed into sub-problems. This observation 
has been implemented in the latest version of \glasso.  In the end of Section 
\ref{sec:quadratic-solver}, we discuss that the {\em fixed}/{\em free} set 
selection can automatically identify the block-diagonal structure of the 
thresholded matrix, and thus \QUIC can benefit from block-diagonal structure 
even when we do not explicitly decompose the matrix.  In the following experiment 
we will show that with input sample covariance $S$ with block-diagonal 
structure represented by $E$ (see Section~\ref{sec:block-diagonal}), \QUIC 
still outperforms \glasso. Moreover, we will show that when some off-diagonal 
elements are added into the problem, while \QUIC is still efficient because of 
its {\em fixed}/{\em free} set selection, \glasso on the other hand suddenly 
becomes very slow.

We generate synthetic data with block-diagonal structure as follows.  We generate a sparse $150\times 150$ inverse covariance matrix $\bar{\Theta}$ as discussed in Section~\ref{sec:synthetic}, and then replicate $\bar{\Theta}$ eight times on the diagonal blocks to form a $1200\times 1200$ block-diagonal matrix. Using this inverse covariance matrix to generate samples, we compare the following methods:
\begin{itemize}
  \item \QUIC: our proposed algorithm.
  \item \glasso: In the latest version of \glasso, the matrix is first decomposed into connected components based on the thresholded covariance matrix $\max(|S|-\lambda)$, and then each sub-problem is solved individually.
  \end{itemize}

We then test the two algorithms for regularization parameter $\lambda$ taking 
values from the set $\{0.017, \dots, 0.011\}$. When $\lambda=0.017$, the thresholded 
covariance matrix $E$ has eight blocks, while when $\lambda=0.011$ the block 
structure reduces to a single block. For each single $\lambda$ trial, we 
compare the time taken by \QUIC and \glasso to achieve 
$(f(X_t)-f(X^*))/f(X^*)<10^{-5}$.  Figure \ref{fig:cluster_exp} shows the 
experimental results.  We can see that both methods are very fast for the case 
where  
the problem can be decomposed into 8 sub-problems (large $\lambda$); however, 
when we slightly increase $\lambda$ so that there is only 1 connected 
component, \QUIC becomes much faster than \glasso.  This is because even for 
the non-decomposable case,
\QUIC can still keep most of the very sparse off-diagonal blocks in the fixed set to speedup the process, 
while \glasso cannot benefit from this sparse structure.

\begin{figure}[ht]
\centering
\includegraphics[width=0.7\textwidth]{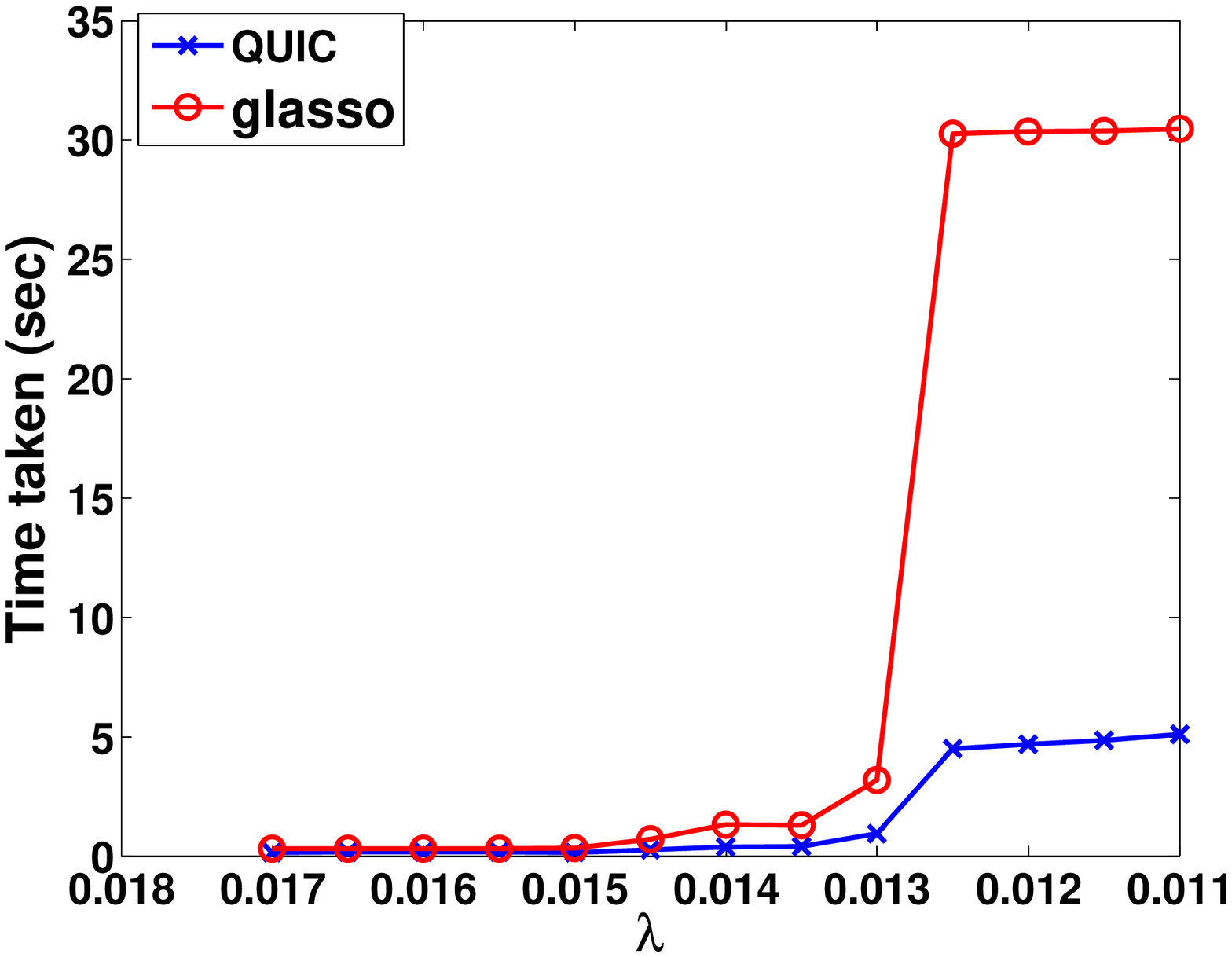}
  \caption{ In this table, we show the performance of \QUIC and \glasso for a sparse synthetic data with clustered structure. Using the same input covariance matrix $S$, we test the time for each algorithm to  achieve $(f(X_t)-f(X^*))/f(X^*)<10^{-5}$ under various values of $\lambda$. When $\lambda=0.017$, the problem can be decomposed into 8 sub-problems, while when $\lambda=0.011$ there is only one component. We can see that for small values of $\lambda$, \QUIC's approach of {\em free}/{\em fixed} set selection is able to exploit the sparse structure of the solution, while \glasso's training time increases drastically.}
  \label{fig:cluster_exp}
\end{figure}

\section*{Acknowledgements}
This research was supported by NSF grants IIS-1018426 and CCF-0728879. ISD acknowledges support from the Moncrief Grand Challenge Award. We would like to thank Kim-Chuan Toh for providing data sets and the \IPM code as well as Katya Scheinberg and Shiqian Ma for providing the \ALM implementation.

\bibliographystyle{plainnat}
\bibliography{QUIC-arxiv}
\end{document}